%% file: iclr2027_conference.tex
\documentclass{article} 
\usepackage{iclr2027_conference,times}

\input{math_commands.tex}

\usepackage[hidelinks]{hyperref}
\usepackage{url}

\usepackage{graphicx}
\usepackage{float}
\usepackage{placeins}
\usepackage{booktabs}
\usepackage{multirow}
\usepackage{wrapfig}


\usepackage{array}
\usepackage{amsmath,amssymb,amsthm}
\usepackage{bbm}
\usepackage{xcolor}
\usepackage{subcaption}
\usepackage{microtype}
\usepackage{tikz}
\usetikzlibrary{patterns,decorations.pathreplacing,calc,positioning}
\input{figures/fig_loop_defs}
\newcommand{\xx}{$\times$}

\newtheorem{proposition}{Proposition}
\newtheorem{property}{Property}
\newtheorem{corollary}{Corollary}

\title{LEAP: Learning Efficient Action Proposals \\ For LLM Agents}

\author{
\parbox{\textwidth}{\centering
\textbf{Zhen Xu}\textsuperscript{1}\qquad
\textbf{Qizheng Zhang}\textsuperscript{2}\qquad
\textbf{Gerry Wan}\textsuperscript{2}\qquad
\textbf{Shang Zhu}\textsuperscript{3}\qquad
\textbf{Ce Zhang}\textsuperscript{1}
\\[0.6em]
{\normalfont\small
\textsuperscript{1}University of Chicago
\qquad
\textsuperscript{2}Stanford University
\qquad
\textsuperscript{3}Together AI
}
}
}

\iclrfinalcopy 

\newlength{\mainfigwidth}
\begin{document}

\maketitle
\lhead{Preprint}

\input{sec_main}
\clearpage

\subsection*{AI use statement}

Generative AI tools were used to assist with English-language editing for readability, development and refinement of analysis code and consistency checks.

\subsection*{Reproducibility statement}

All models and benchmarks used in this study are open sourced. Detailed experimental setups and results are provided in Section~\ref{sec:experiments} and the appendices. We plan to release the code and experimental artifacts on GitHub.

\bibliography{refs}
\bibliographystyle{iclr2027_conference}

\clearpage
\appendix
\input{sec_details}

\input{sec_appendix}

\end{document}

%% file: math_commands.tex
\usepackage{amsmath,amsfonts,bm}

\def\eqref#1{equation~\ref{#1}}

\def\1{\bm{1}}

\DeclareMathAlphabet{\mathsfit}{\encodingdefault}{\sfdefault}{m}{sl}
\SetMathAlphabet{\mathsfit}{bold}{\encodingdefault}{\sfdefault}{bx}{n}



%% file: figures/fig_loop_defs.tex
\definecolor{genfill}{RGB}{203,213,228}\definecolor{genline}{RGB}{78,98,132}
\definecolor{draftfill}{RGB}{240,226,200}\definecolor{draftline}{RGB}{168,128,70}
\definecolor{toolline}{RGB}{127,127,127}\definecolor{userfill}{RGB}{234,234,234}
\definecolor{cancelline}{RGB}{110,110,110}\definecolor{panelbg}{RGB}{245,245,245}
\definecolor{savedfill}{RGB}{214,238,218}\definecolor{savedline}{RGB}{35,120,70}
\definecolor{extrafill}{RGB}{249,216,211}\definecolor{extraline}{RGB}{180,45,35}
\tikzset{
  gen/.style={draw=genline, fill=genfill, line width=0.3pt},
  draftbar/.style={draw=draftline, fill=draftfill, line width=0.3pt},
  toolbar/.style={draw=genline, pattern=north east lines, pattern color=toolline, line width=0.3pt},
  cancelbar/.style={draw=cancelline, dashed, fill=none, line width=0.4pt},
  userbar/.style={draw=genline, fill=userfill, line width=0.3pt},
  tlab/.style={font=\scriptsize, inner sep=0.5pt},
  rowlab/.style={font=\footnotesize, anchor=east, inner sep=2pt},
  note/.style={font=\scriptsize, text=black!75, inner sep=1pt},
  ptag/.style={font=\scriptsize\bfseries, anchor=north west, inner sep=1pt},
  savedband/.style={fill=savedfill, fill opacity=0.75, draw=none},
  extraband/.style={fill=extrafill, fill opacity=0.75, draw=none},
  fbox/.style={draw, rounded corners=2pt, align=left, inner sep=3.5pt, font=\footnotesize, line width=0.4pt},
  draftbox/.style={fbox, draw=draftline, fill=draftfill},
  targetbox/.style={fbox, draw=genline, fill=genfill},
  plainbox/.style={fbox, draw=black!60, fill=black!4},
  flow/.style={->, >=stealth, line width=0.5pt, draw=black!70},
  flowlab/.style={font=\scriptsize, text=black!75, inner sep=1.5pt, fill=panelbg},
}
\newcommand{\loopband}[5]{\fill[#3] ({#1},-0.02) rectangle ({#2},2.41);
  \draw[#4, line width=0.5pt] ({#1},-0.02) -- ({#1},2.41); \draw[#4, line width=0.5pt] ({#2},-0.02) -- ({#2},2.41);
  \node[font=\footnotesize\bfseries, text=#4, anchor=south] at ({(#1+#2)/2},2.44) {#5};}
\newcommand{\loopbar}[6]{\draw[#4] ({#1},{#3}) rectangle ({#2},{#3+#5});\ifx&#6&\else\node[tlab] at ({(#1+#2)/2},{#3+#5/2}) {#6};\fi}
\newcommand{\loopframe}[1]{%
  \useasboundingbox (-1.75,-0.3) rectangle (11.4,2.8);
  \draw[black!35, line width=0.3pt, ->] (0,-0.05) -- (11.2,-0.05); \node[note, anchor=north east] at (11.3,-0.08) {time};
  \looptagnode{#1}}
\newcommand{\looptagnode}[1]{\node[ptag] at (-1.75,2.95) {#1};}
\newcommand{\looptagA}[1]{\node[font=\bfseries\scriptsize, anchor=south west] at ([yshift=2pt]s.north west) {#1};}

%% file: sec_main.tex
\begin{abstract}
LLM agents are known to be slow in rollouts. An agent completes a task one step at a time. At each step, it reasons and then chooses an action to execute. The next step and action cannot start until the previous one has finished.  
Speculative decoding accelerates the rollouts at the reason phase by drafting and verifying the inference tokens.
Recent works have also started to apply similar ideas at the action phase. These works use off-the-shelf models, usually large, to draft action proposals for target model to verify.
Large drafters match the target more often but take longer to propose, while small off-the-shelf models are fast but rarely make the same decision as the target.
We ask a more general question: what determines the end-to-end speedup of action speculation?
To answer it, we develop a latency framework for the speculative round. The framework compares what a round gains with what it costs. The gain depends on how well the drafter predicts the target and on how many steps the task can take before it ends. The cost comes from drafting, from waiting for target verification and from executing tools.
Guided by the framework, we introduce \textbf{LEAP} (\textbf{L}earning \textbf{E}fficient \textbf{A}ction \textbf{P}roposals) which keeps the drafter small and makes it accurate by training it on the target actions sequences.
With a small 0.6B model, LEAP agrees with the target on most decisions and makes agents up to 60\% faster in end-to-end wall clock time, with no systematic change in task success.
Across various datasets, target models and draft models, the framework accounts for most of the measured speedups.
We also show the draft model can be online trained with no prior trace collection and match the performance of offline training, making LEAP practical to deploy in the real world.
\end{abstract}

\section{Introduction}
An LLM agent solves a task through successive decisions, such as selecting a tool and supplying its
arguments. Each decision requires a full decode, often including reasoning, and later
decisions depend on earlier actions and observations. For our 32B target on one H100, a
$\tau^2$-bench task takes about three minutes, nearly all of it generation. Improving the latency of
this sequential loop matters even when its individual calls already use an efficient serving engine.

Speculative decoding reduces generation latency within a step. A cheap draft model proposes tokens
for an expensive target to verify~\citep{leviathan2023,chen2023}. Recent agent systems apply a similar
idea to whole actions. Their proposals come from off-the-shelf models, cheaper prompting configurations,
or the target itself~\citep{isp,dsp,specactions,spork,swyr}. These choices expose a tradeoff.
An untrained small model is cheap but may rarely agree with the target. A large self-draft agrees
more often but can spend much of the time that speculation was meant to save.

We ask what determines the end-to-end speedup of an action-speculation round. The draft proposes
up to $k$ actions, and the target verifies the resulting contexts concurrently. The round commits
the matching prefix and a target-generated correction or bonus unless the task ends.
Our latency framework compares the committed steps gained with the cost of proposing, waiting for
verification and executing tools. Efficient proposals must be likely to be accepted and cheap relative
to a target step. Acceptance alone does not determine speedup: a larger draft or the target itself
can be more accurate yet slower. Task endings further limit how many steps a round
can commit.

Guided by this framework, \textbf{LEAP} (Learning Efficient Action Proposals) keeps the draft small
and trains it to predict the target's canonical actions. Every target-authored commit supplies a
training pair, so supervision is a by-product of running the agent. Fine-tuning a LoRA adapter without
reproducing the reasoning text raises a 0.6B draft's agreement with Qwen3-32B from 2--25\% to 64--75\%
across our benchmarks (Table~\ref{tab:offline_main}). The draft proposes, the target decides, and both
share one GPU during inference.

The paper makes three contributions. First, the latency framework relates speedup to round gains
and costs. With measured hit rates and draft costs, its predictions are within 0.10 of measurement
on 19 of 21 configurations. Predictions using offline hit rates meet that threshold on 12 of 18
configurations; some use measured draft latency. Second, LEAP trains a tiny draft from the target's
own actions, either offline or online. Online inference starts near sequential per-step speed and
approaches the offline draft without a prior collection pass. Third, we evaluate four benchmarks,
two target families and three draft sizes on a shared GPU. The trained 0.6B gives 1.10--1.63\xx{}
whole-task speedup, up to about 60\% faster. Across 16 tool-benchmark comparisons, official task
success increases in eight, decreases in five and is unchanged in three, with differences of at
most three tasks. Planning termination counts are reported separately. Every committed action
passes the target-authorship audit. Figure~\ref{fig:loop} summarizes training and execution.

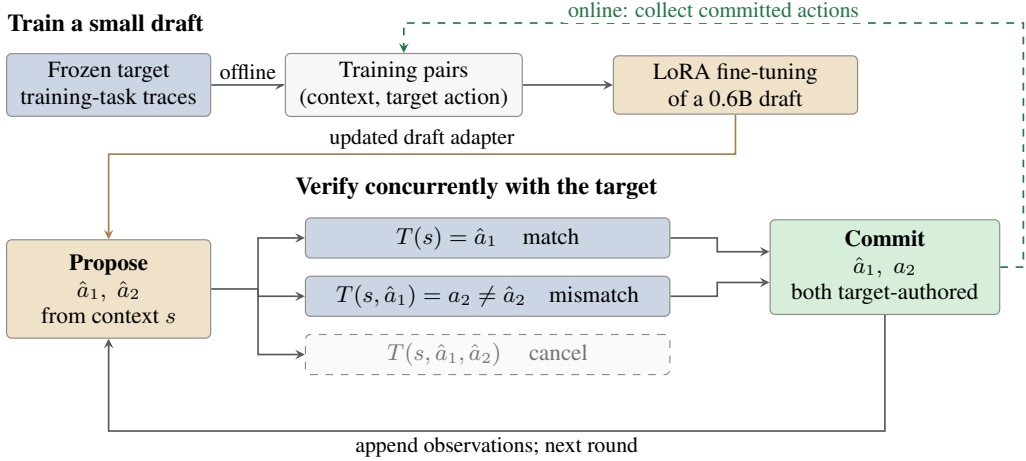
\begin{figure}[tbp]
\centering
\resizebox{\mainfigwidth}{!}{\input{figures/fig_loop_compact}}
\caption{\textbf{LEAP trains and uses an action draft.} Offline target traces or online commits supply
the training pairs. The adapter proposes actions; the target verifies their contexts concurrently.
In this $k=2$ example, one proposal matches and the target corrects the next. The remaining
verification is cancelled, and the two committed actions can become further training data.}
\label{fig:loop}
\end{figure}

\section{Related work}
\label{sec:related}
\paragraph{Token-level speculation.}
Speculative decoding preserves the target's token distribution while verifying draft tokens in
parallel~\citep{leviathan2023,chen2023}. Drafts may be separate models or prediction heads such as
Medusa and EAGLE~\citep{medusa,eagle}; DistillSpec and domain-adapted drafts train for higher
acceptance~\citep{distillspec,domaindraft}. Online speculative decoding updates a token draft using
deployment queries~\citep{osd}. OoO-Spec applies speculative decoding to tool calls~\citep{ooospec}. LEAP instead imitates canonical agent actions. 

\paragraph{Action and tool speculation.}
Interactive and Dynamic Speculative Planning use cheaper prompting configurations~\citep{isp,dsp};
Speculative Actions studies action-level verification and depth~\citep{specactions}; other systems
use untrained small models or modified target configurations~\citep{smc,spagent,dualspec}.
SPORK forks a probe from the target's state, while Speculate While You Reason trains a
self-speculation head jointly with the agent~\citep{spork,swyr}. We train a separate small draft with
the target frozen and measure both models together on one GPU.
Tool-latency methods launch predicted calls while generation continues~\citep{stc,paste,swm,aospec}.
They address a different cost from the sequential target-generation waits reduced by LEAP; our
proposals still pay tool-execution time. Additional comparisons are in App.~\ref{app:related}.

\section{Designing efficient action proposals}
\subsection{Propose, verify, commit}
\label{sec:problem}\label{sec:loop}
The context $s_i$ contains the instruction, tool schemas, and committed actions and
observations. The target greedily generates a canonical action $a_i=T(s_i)$; the draft predicts it
without reasoning. Actions must match after parsing and canonicalization, including tool names
and arguments; invalid proposals never match. Each synchronous round follows the propose--verify--commit pattern~\citep{specactions}
(Fig.~\ref{fig:loop}):
\begin{enumerate}
\item \textbf{Propose.} Generate up to $k$ actions, conditioned on earlier proposals and their
observations. Proposed tools execute in a copy of the environment.
\item \textbf{Verify.} The target generates actions on the resulting contexts, up to $k+1$,
concurrently on one shared GPU.
\item \textbf{Commit.} Consume verifications in prefix order. Commit the $A$ leading matches and
the target's correction or bonus, cancelling unused requests. Adopt accepted environment states
without executing their tools again.
\end{enumerate}
A round commits $A+1$ actions unless its accepted prefix completes the task, which leaves no
bonus. Proposing finishes before verification.

\begin{property}[target authorship]
\label{prop:exact}
Every committed action equals the target's verification output on the context in which that action
is committed.
\end{property}
The guarantee follows from matching before commit. Concurrent decoding can still change the target's
decisions, so trajectory identity and task outcomes are evaluated separately (App.~\ref{app:success}).

\subsection{A latency framework}
\label{sec:model}
The framework compares the steps gained with the time spent in a round, for any draft used in this loop.
Let $q_j=\Pr[A\ge j]$ include rounds that end before position $j$, and let $e$ be the probability that the
accepted prefix completes the task, which removes the bonus. The mean committed steps per round are
$c=1+\sum_{j=1}^{k}q_j-e$.

Let $t_{\mathrm{gen}}$ be the mean sequential target-generation time. Write $\beta t_{\mathrm{gen}}$
for the mean longest verification time actually awaited, and $D$ for mean proposing time per round
divided by $t_{\mathrm{gen}}$, including proposal-tool execution. With relative tool cost
$\tau=t_{\mathrm{tool}}/t_{\mathrm{gen}}$, let $f$ be the fraction of sequential steps executing a tool
and $g$ the mean number of target-executed tool steps per round. The speedup is approximately
\begin{equation}
 S\simeq\frac{c(1+f\tau)}{\beta+D+g\tau}
 \quad\xrightarrow{\;\tau\to0\;}\quad\frac{c}{\beta+D}.
 \label{eq:main:S}
\end{equation}
Accepted proposal tools are included in $D$ only once. This model assumes comparable work between
arms, small bookkeeping costs and cancellation of discarded verifications.

Efficient proposals are therefore likely to be accepted and cheap relative to a target step.
Pooled acceptance alone does not determine speedup: it omits rounds with no usable proposal and
does not encode task endings or proposing costs. Increasing depth can raise both commits
and round cost; the bounds and depth condition are in App.~\ref{app:model}.

\paragraph{Estimating the inputs.}
Estimate $c$ by evaluating the draft on held-out sequential target contexts and replaying the acceptance
rule and task endings. Estimate $\beta$ by replaying sequential timings with estimated hits and taking
each round's maximum awaited step time. Estimate $D$ from single-stream draft-latency
probes and the expected proposal count, plus proposal-tool costs (App.~\ref{app:predictor}).
Concurrency-dependent changes in target outputs and timing remain sources of prediction error.

\subsection{LEAP: training the draft}
\label{sec:train}
Guided by the framework, LEAP keeps the draft small and trains it on context--action pairs generated
by the frozen target. We fine-tune a LoRA adapter with cross-entropy on canonical action tokens,
discarding reasoning text. No target logits or human action labels are needed. The draft decodes
without reasoning and uses schema constraints on tool benchmarks; training settings are in
App.~\ref{app:detail:train}.

Only proposals conditioned on a fully matched prefix can be committed (Proposition~\ref{prop:commit}),
motivating supervised imitation on target traces. Property~\ref{prop:exact} also makes deployed
commits a source of online training pairs, without a separate collection pass.

\section{Experiments}
\label{sec:experiments}
\subsection{Setup and main results}
\label{sec:setup}\label{sec:results4}
We evaluate Qwen3-32B and Gemma-4-31B-it, each with trained Qwen3 drafts of 0.6B, 1.7B and 4B.
The whole pipeline is repeated for each target: its own sequential traces supply the training data.
Target and draft share one H100 during inference. Serving configurations are held fixed across
sizes within each benchmark; the Gemma target uses FP8 weights (App.~\ref{app:serving}).

We use four benchmarks. OpenAGI and TaskBench generate plans without executing tools~\citep{openagi,taskbench}.
$\tau^2$-bench uses the telecom domain in solo mode, with tools executed against its local backend;
BFCL uses multi-turn tasks and cloned tool environments~\citep{tau2bench,bfcl}.
Train/test splits are task-disjoint. Each task runs sequentially and speculatively on
the same service. We report the ratio of total sequential to speculative wall-clock. The main
speedup tables show the estimate $\pm$ one standard error (SE), computed as the standard deviation
of 10,000 task-level bootstrap ratios. Sequential and speculative times stay paired, and repeated
runs of a task remain together when resampled. Full 95\% intervals are retained in the appendix.

\input{tab_main_results}

Table~\ref{tab:main_results} covers all three sizes and both targets. The trained 0.6B yields
1.18--1.61\xx{} with Qwen3 and 1.10--1.63\xx{} with Gemma. The benchmarks differ in the terms of
Eq.~\ref{eq:main:S}. OpenAGI accepts many proposals but its short tasks frequently remove the bonus.
$\tau^2$-bench has longer tasks and a cheap draft relative to its target generation time. BFCL
tasks are multi-turn conversations; a round cannot cross the end of a turn, and turns average only
2.6 steps, limiting commits per round.
Changing the target also changes the relative cost: Gemma's planning steps are longer, while its
tool-benchmark steps are shorter, so the same draft consumes a different fraction of a target step.

Every committed action passes the target-authorship audit. Across 16 comparisons on $\tau^2$-bench
and BFCL, official task success increases in eight, decreases in five, and is unchanged in three.
Differences are at most three tasks per comparison. The 12 planning-benchmark comparisons count
normal termination: three decrease and nine are unchanged (App.~\ref{app:success}). Trajectories
can still diverge. We report these outcomes separately from the per-action guarantee and retain
all tasks in the wall-clock estimator.

\subsection{Prediction: agreement with the target}
\label{sec:agreement}
We first examine how accurately a small draft predicts the target, which determines the
accepted-prefix contribution to a round's gain.
Table~\ref{tab:offline_main} compares draft agreement under a common evaluation setup. Training
substantially improves the small base. A reasoning-free target is a strong predictor on planning
tasks, but on $\tau^2$-bench it agrees on 42.5\% of contexts versus 63.7\% for the trained 0.6B.
Disabling reasoning can change the decision, so the target's size alone does not guarantee a better
action prediction. The reasoning-on target is also shown as a self-consistency reference.

\input{tab_main_agreement}

The SEs in Table~\ref{tab:offline_main} resample whole tasks and retain every context and repeated
prediction of each sampled task. Each replicate recomputes the context-weighted agreement, accounting
for dependence between decisions from the same trajectory. These SEs quantify task-sampling uncertainty
conditional on the evaluated models and runs; they do not quantify variation across training seeds.
BFCL has 99 contexts from ten tasks; its trained-draft SE is 6.0 percentage points. The $\tau^2$-bench
reasoning-on reference uses the per-context mode across three passes.

Task endings limit how often later proposals can be used (Sec.~\ref{sec:model}).
On OpenAGI, the second and third proposals are made in only 81\% and 62\% of rounds.
Among rounds that reach a position with the earlier proposals accepted, acceptance is about the
same as or higher than at the first position on every benchmark (Table~\ref{tab:agree}).

\paragraph{Supervision and objective.}
Table~\ref{tab:objective} tests learning the target's actions against benchmark gold supervision and
GRPO from the SFT draft. Gold-label training gives lower agreement on TaskBench and BFCL. The
comparison changes the training contexts as well as the labels: gold contexts follow gold actions,
and BFCL's target can group calls. The comparison therefore assesses the complete supervision source.
GRPO gives a small TaskBench gain and no BFCL gain, at substantially more
training cost; no variant exceeds the three-point agreement gain used to trigger an end-to-end run.

\input{tab_main_objective}

\paragraph{Data and generalization.}
The OpenAGI recipe study separates the 40 test tasks whose action sequences also occur in training
from 23 with unseen sequences (Table~\ref{tab:recipe}). The default adapter obtains 82.2\% agreement
on the former and 64.5\% on the latter. Selecting epoch 7 by validation exact match raises overall
agreement from 75.0\% to 77.1\%, but the unseen score is 63.2\%. Full fine-tuning at epoch 7 gives
75.2\% overall, below the rank-64 adapter. Additional capacity or epochs therefore improve neither
subset uniformly in this study.

Using the same epoch-7 recipe with half or one quarter of the training pairs gives 72.0\% and
71.5\% overall agreement, versus 77.1\% with the full data. A few hundred pairs already supply
useful supervision, although performance depends on action-sequence overlap and the unseen score
is not monotone in data quantity. Cross-benchmark transfer is also limited: the TaskBench draft
reaches 23.3\% on OpenAGI and the OpenAGI draft 29.6\% on TaskBench, far below their in-benchmark
rates of about 75\% (Table~\ref{tab:transfer}). Most other transfers stay near the untrained base.
These results motivate workload-specific adapters and continued collection of the target's actions.
We retain action-only input because the target's reasoning is unavailable for later speculative
positions when the draft proposes them; full input and recipe studies are in App.~\ref{app:training}.

\subsection{Cost: when a more accurate draft is slower}
\label{sec:selfspec}\label{sec:training}
We next test whether extra agreement offsets the cost of proposing. The framework compares
the resulting gain in committed steps with the change in round time.

Figure~\ref{fig:curve} shows measured acceptance and drafting cost for the trained and untrained
0.6B drafts and the reasoning-disabled target. The untrained TaskBench draft produces no usable
proposal in 38\% of rounds; these rounds cost drafting time without entering the acceptance ratio
(Sec.~\ref{sec:model}).

\begin{figure}[tbp]
\centering
\includegraphics[width=\mainfigwidth]{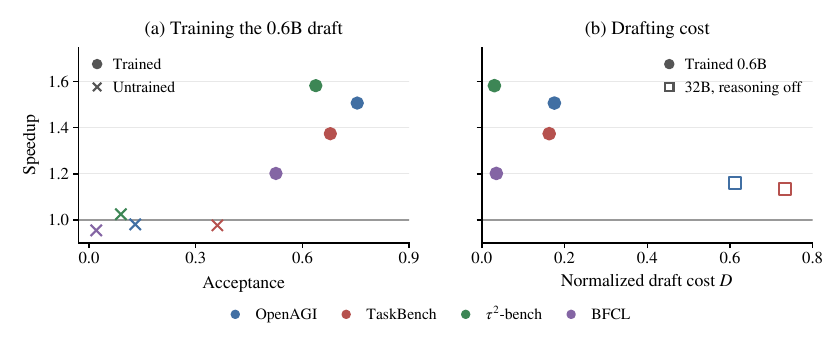}
\caption{\textbf{Measured acceptance and drafting cost, Qwen3 target.} (a) Trained and untrained
0.6B controls. Acceptance is the share of compared proposals that match the target. (b) Trained
0.6B on all four benchmarks, plus the reasoning-disabled target on the two planning benchmarks
where this control was run; $D$ is proposing time per round normalized by a target step. All points use the control runs of Table~\ref{tab:model}.}
\label{fig:curve}
\end{figure}

The planning controls use the target with reasoning disabled as the proposer in our same
propose--verify--commit loop. Their live acceptance is comparable to or above the trained draft's,
yet Qwen3 self-drafting gives only 1.16\xx{} and 1.13\xx{} on OpenAGI and TaskBench
(Table~\ref{tab:model}). Each proposal costs $3.3$--$4.0\times$ as much, raising the normalized draft
term from 0.16--0.18 to 0.61--0.73. With Gemma's longer planning steps, self-drafting instead reaches
1.48\xx{} and 1.54\xx{}: drafting cost is smaller relative to a target step. These measurements
illustrate why acceptance and relative cost must be considered together.

\paragraph{Draft size.}
Figure~\ref{fig:size} shows the draft-size ablation. On Qwen3 planning tasks, agreement generally
increases with size, but a 4B draft is slower than 1.7B and returns to about the 0.6B speedup. Its
per-round cost grows from 0.16--0.17 to 0.28--0.29 target steps, offsetting the extra commits.
On BFCL, greater agreement carries through to higher speedup because drafting is cheap relative to
the target step. On $\tau^2$-bench the three sizes' 95\% intervals overlap (Table~\ref{tab:four}). Gemma changes this balance
again: its shorter tool-benchmark generations increase the relative draft cost, and its 0.6B draft
is fastest on $\tau^2$-bench (Table~\ref{tab:main_results}).

\begin{figure}[tbp]
\centering
\includegraphics[width=\mainfigwidth]{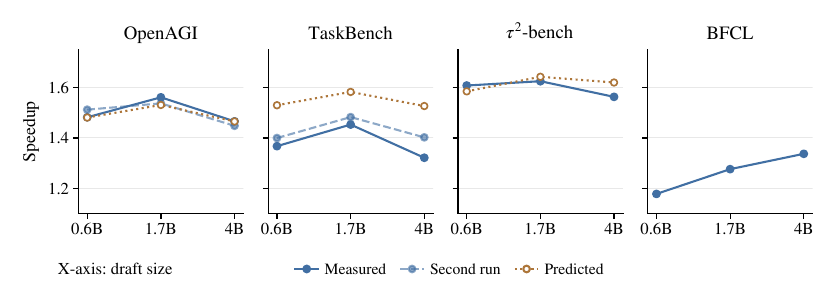}
\caption{\textbf{Draft-size ablation, Qwen3 target.} Measured speedup, independent planning repeats,
and the predictions used in Fig.~\ref{fig:pred}b. BFCL has one measured size study and no offline-hit
prediction. Agreement and cost values are in Table~\ref{tab:size}; pooled results are in
Table~\ref{tab:main_results}.}
\label{fig:size}
\end{figure}

The cost decomposition makes this comparison quantitative. In the generation-dominated regime,
the model compares two draft configurations through
\begin{equation}
 \frac{S_B}{S_A}\simeq
 \frac{c_B/c_A}{(\beta_B+D_B)/(\beta_A+D_A)}. \label{eq:main:size_ratio}
\end{equation}
For the measured Qwen3 components, moving from 0.6B to 4B raises commits per round by a factor
of 1.09--1.10 on the planning benchmarks but raises round cost by 1.12, explaining the loss in
speedup. On BFCL the corresponding factors are 1.16 and 1.03, so the larger draft pays off.
On $\tau^2$-bench both ratios are within 0.02 of one. These are explanations using measured
components; predictive comparisons use estimated components in the same ratio. Their resolution
depends on the relative input errors, particularly when expected speedups differ by only a few
hundredths. The full confidence intervals and repeated planning runs provide the relevant empirical
context for these point-estimate comparisons.

The 0.6B draft is within 0.1\xx{} of the fastest size in seven of eight benchmark--target pairs.
For practitioners, we recommend starting with a trained 0.6B draft and $k{=}1$.
The main runs use the benchmark settings in Table~\ref{tab:main_results}; depth studies and their
uncertainty are reported in Table~\ref{tab:model}. The recommendation prioritizes low drafting and
training cost.

\noindent
\begin{minipage}[t]{0.56\linewidth}
\vspace{0pt}
\textbf{Tool latency.}
Adding 2 or 5 seconds to every executed tool step in both arms tests the tool terms. At 5 seconds,
measured speedup falls to 1.20\xx{} on $\tau^2$-bench and 1.02\xx{} on BFCL; trace-replay $\beta$
predicts 1.17\xx{} and 1.04\xx{} using measured hit rates and proposing costs (Table~\ref{tab:tool}).
Agreement changes little, but proposals pay tool latency while rolling out. The normalized proposing
cost grows from 0.03 to 0.31 and 0.74 on $\tau^2$-bench, and from 0.04 to 0.23 and 0.50 on BFCL.
On BFCL, turn ends execute no tool, so $f=0.62$ and $g=0.44$--$0.45$ (App.~\ref{app:predictor}).
\end{minipage}\hfill
\begin{minipage}[t]{0.41\linewidth}
\vspace{0pt}
\input{tab_tool_latency}
\end{minipage}
\par\medskip
A round also pays for the target's own tool action, while accepted proposal states are adopted.
These costs account for the loss of speedup despite nearly unchanged hit rates. The four added-delay
predictions are within about 0.08\xx{} of measurement; using the measured verification wait reduces
the residual to about 0.02\xx{}. This diagnostic supports the tool-cost accounting while identifying
wait estimation as a remaining error source.

\subsection{Checking the latency framework}
\label{sec:validation}
Finally, we test the framework's combined accounting of round gains and costs against measured speedup.
We first validate the accounting with the run's measured hit rates and draft cost, while predicting
$\beta$ from its sequential timing trace. Across 21 configurations spanning depths, untrained
drafts, self-drafting and added tool latency, predicted speedup is within 0.05\xx{} on 14 and
within 0.10\xx{} on 19 (Fig.~\ref{fig:pred}a). Component checks and full inputs are in
Fig.~\ref{fig:components} and Table~\ref{tab:model}.

\begin{figure}[tbp]
\centering
\includegraphics[width=\linewidth]{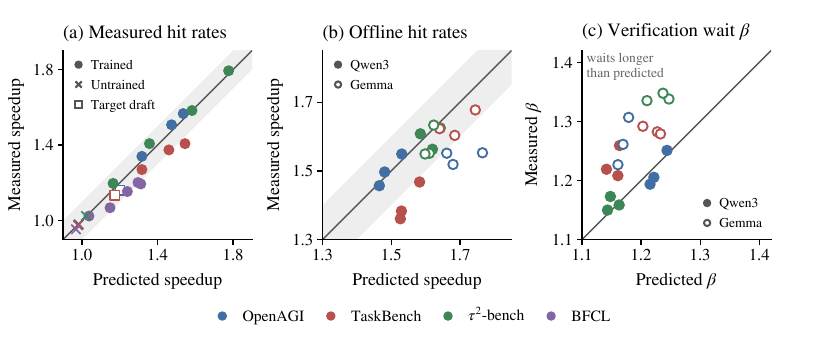}
\caption{\textbf{Predicted versus measured speedup and verification wait.} (a) 21 Qwen3 controls with measured hit rates
and draft cost. (b) Nine size-study configurations per target, using offline hit rates and trace-replay
$\beta$; draft timing sources are in Table~\ref{tab:prediction_inputs}. (c) Verification wait $\beta$ for the
points of (b): the trace-replay prediction used in (b) against the measured value. Points above the
diagonal waited longer than predicted, which raises the predicted speedup. Axes have equal scales within
each panel; shading marks $\pm0.10$ in (a) and (b).}
\label{fig:pred}
\end{figure}

We then replace live hit rates with the draft's predictions on held-out sequential contexts and
replay the commitment rules without live acceptance measurements. In the Qwen3 control runs, offline draft-latency probes
and trace-replay $\beta$ predict 1.48, 1.56 and 1.59\xx{} on OpenAGI, TaskBench and $\tau^2$-bench,
against 1.51, 1.37 and 1.58\xx{} measured. Figure~\ref{fig:pred}b
applies the same replay to all three draft sizes in the serving configurations of Table~\ref{tab:main_results},
giving nine points per target. The 1.7B and 4B Qwen3 planning drafts and the larger Gemma drafts use
measured per-proposal latency, so these predictions are conditional on that cost. Twelve of the eighteen
size-study predictions are within 0.10\xx{}. The control and size studies use their respective serving
configurations; their 0.6B measurements need not coincide.

The errors differ by target. The trace replay underestimates the verification wait $\beta$ on every
Gemma configuration, by 0.05--0.13, and on Qwen3 TaskBench, by 0.05--0.10; on Qwen3 OpenAGI and
$\tau^2$-bench it is within 0.03 (Fig.~\ref{fig:pred}c). A lower $\beta$ raises the predicted speedup:
all six points that miss by more than 0.10 in (b) have an underestimated $\beta$. The cause of this
gap is not established. On Qwen3 TaskBench the replay also overpredicts the first-position hit rate
(0.73 versus 0.66 in the control run), because live trajectories depart from the sequential trace more
often than on OpenAGI. On Gemma $\tau^2$-bench, the replay's lower hit rates offset the $\beta$ gap, so a
small total speedup error can reflect cancellation between input errors. These effects limit reliable
ranking of nearby configurations. BFCL is not in (b) or (c) because it has no offline-hit prediction
(App.~\ref{app:detail:validation}).

\FloatBarrier
\subsection{Learning the draft online}
\label{sec:online}
The separate trace-collection pass is optional: committed actions are target-authored even when the
draft is initially untrained. We start the speculative system with the 0.6B base and process the
training tasks in batches: 25 tasks on OpenAGI and 10 on $\tau^2$-bench. After each batch, an adapter
is trained on all commits so far on a second GPU and loaded into the draft server. Batch $b$ uses
only the adapter trained on earlier batches. A continual variant instead updates the previous
adapter on the new batch plus an equal-size sample of earlier pairs.

\begin{figure}[htbp]
\centering
\includegraphics[width=\mainfigwidth]{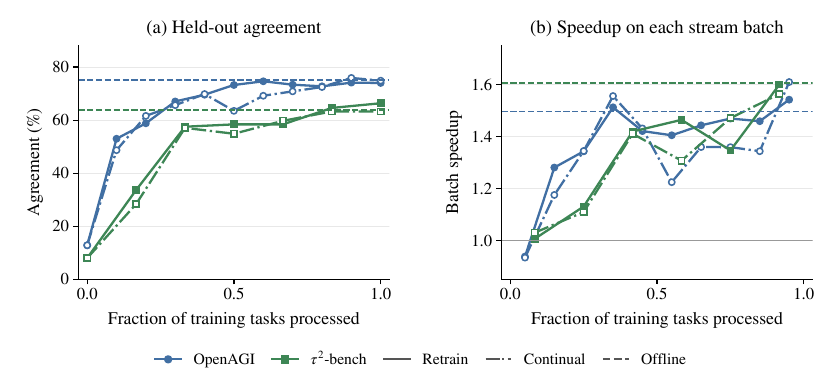}
\caption{\textbf{Learning online, Qwen3 target.} (a) Held-out agreement after each adapter update.
(b) Speedup on each stream batch, placed at the batch midpoint: sequential time per step over
speculative time per committed step. Dashed lines are offline drafts. Each batch uses an adapter
trained on earlier tasks. Training time is reported separately; final test results are in
Table~\ref{tab:online_final}.}
\label{fig:online}
\end{figure}

Agreement grows rapidly (Fig.~\ref{fig:online}a). On OpenAGI, retraining reaches 53\% after 25 tasks,
67\% after 75, and 74.0\% after 250, against the offline draft's 75.0\%. Its batch speedup rises
from 0.94\xx{} initially to 1.4--1.5\xx{} from the fourth batch. On $\tau^2$-bench, agreement rises
from 8.0\% to 66.4\%, versus 63.7\% offline, and the batch speedup rises from 1.01\xx{} to 1.4\xx{}
after 20 tasks and 1.60\xx{} in the last batch (Fig.~\ref{fig:online}b). The final adapters attain
similar test speedups to the offline drafts (Table~\ref{tab:online_final}).

\input{tab_online_results}

Retraining on the growing buffer costs 27--73 seconds per update on OpenAGI and 3--18 minutes on
$\tau^2$-bench. Continual updates take about 22 seconds and 3--6 minutes respectively. They reach
similar endpoints, though OpenAGI dips once before recovering. Update times are reported separately
from inference speedups.
Version-by-version agreement and update timings are in App.~\ref{app:online}.

\section{Conclusion and Limitations}
\label{sec:conclusion}
We asked what determines end-to-end speedup in action speculation. The latency framework compares
the committed steps gained with proposing, verification and tool costs. It identifies the design
goal of proposals that are likely to be accepted and cheap relative to a target step. LEAP meets
this goal by training a tiny draft on the frozen target's actions. Across four benchmarks and two
target families, the trained 0.6B gives 1.10--1.63\xx{} whole-task speedup on one shared GPU.
The same supervision is available during deployment, allowing online training without a prior
trace-collection pass.

We evaluate two target families across various benchmarks, using local
simulated tool backends. The framework covers synchronous proposing followed by exact-match
verification, with target contexts verified concurrently on one GPU. It does not cover asynchronous
draft run-ahead, semantic verification or methods that hide tool latency through overlap.
The loop trades target compute for latency: verification can issue up to
$k+1$ requests per round, and $\tau^2$-bench uses 1.16\xx{} the sequential completion tokens in
the measured control. Slow tools reduce its benefit because proposing pays their execution cost.
Concurrent decoding can change whole trajectories despite target authorship. The predictor assumes
comparable work and transfer of sequential timing to concurrent requests; these approximations and
the bootstrap intervals limit close configuration comparisons. Online training could consume additional
resources depending on the actual setup.
\label{main:text_end}

%% file: figures/fig_loop_compact.tex
\begin{tikzpicture}[x=1cm,y=1cm,>=stealth,
 every node/.style={font=\fontsize{9}{10.5}\selectfont,inner sep=3pt},
 box/.style={draw=black!45,rounded corners=2pt,align=center,line width=.45pt},
 arr/.style={->,draw=black!65,line width=.65pt},
 tag/.style={font=\fontsize{8.5}{10}\selectfont,fill=white,inner sep=2pt}]
\useasboundingbox (-.1,.25) rectangle (14.3,6.55);
\node[anchor=west,font=\fontsize{9.5}{11}\selectfont\bfseries] at (0,6.25) {Train a small draft};
\node[box,fill=genfill,minimum width=2.8cm,minimum height=.8cm] (traces) at (1.5,5.4)
 {Frozen target\\training-task traces};
\node[box,fill=black!3,minimum width=3.25cm,minimum height=.8cm] (pairs) at (5.55,5.4)
 {Training pairs\\(context, target action)};
\node[box,fill=draftfill,minimum width=3.35cm,minimum height=.8cm] (train) at (10.1,5.4)
 {LoRA fine-tuning\\of a 0.6B draft};
\draw[arr] (traces) -- node[tag,above]{offline} (pairs);
\draw[arr] (pairs) -- (train);
\draw[arr,draw=draftline] (train.south) -- (10.1,4.45) -- node[tag,above]{updated draft adapter} (1.5,4.45) -- (1.5,3.28);
\node[anchor=west,font=\fontsize{9.5}{11}\selectfont\bfseries] at (3.95,3.98) {Verify concurrently with the target};
\node[box,fill=draftfill,minimum width=2.8cm,minimum height=1.35cm] (draft) at (1.5,2.6)
 {\textbf{Propose}\\$\hat a_1,\;\hat a_2$\\from context $s$};
\node[box,fill=genfill,minimum width=5.0cm,minimum height=.55cm] (v1) at (6.7,3.3)
 {$T(s)=\hat a_1$ \hspace{.6em} match};
\node[box,fill=genfill,minimum width=5.0cm,minimum height=.55cm] (v2) at (6.7,2.5)
 {$T(s,\hat a_1)=a_2\ne\hat a_2$ \hspace{.5em} mismatch};
\node[box,dashed,fill=black!2,text=black!55,minimum width=5.0cm,minimum height=.55cm] (v3) at (6.7,1.7)
 {$T(s,\hat a_1,\hat a_2)$ \hspace{.7em} cancel};
\draw[arr] (draft.east) -- (3.55,2.6) |- (v1.west);
\draw[arr] (3.55,2.6) |- (v2.west);
\draw[arr] (3.55,2.6) |- (v3.west);
\node[box,fill=savedfill,minimum width=3.15cm,minimum height=1.3cm] (commit) at (12.15,2.9)
 {\textbf{Commit}\\$\hat a_1,\;a_2$\\both target-authored};
\draw[arr] (v1.east) -- (9.85,3.3) |- ([yshift=6pt]commit.west);
\draw[arr] (v2.east) -- (9.85,2.5) |- ([yshift=-6pt]commit.west);
\draw[arr] (commit.south) -- (12.15,.65) -- node[tag,below]{append observations; next round} (1.5,.65) -- (draft.south);
\draw[arr,draw=savedline,dashed] (commit.east) -- (14.05,2.9) -- (14.05,6.25) --
 node[tag,above,text=savedline]{online: collect committed actions} (5.55,6.25) -- (pairs.north);
\end{tikzpicture}

%% file: tab_main_results.tex
\begin{table}[htbp]
\centering\small\setlength{\tabcolsep}{6pt}\renewcommand{\arraystretch}{1.07}
\caption{\textbf{Whole-task speedup across benchmarks, targets and draft sizes.} Values are pooled
speedup factors $\pm$ one task-bootstrap SE (10,000 paired resamples); bold marks the largest rounded
point estimate per row. $k=3$ for planning and $k=2$ for tool benchmarks; test sizes are 63/60/40/60.
Qwen3 planning entries pool two runs. Full 95\% intervals and components: Table~\ref{tab:four}.}
\label{tab:main_results}
\begin{tabular}{llccc}
\toprule
Benchmark & Target & 0.6B & 1.7B & 4B \\
\midrule
OpenAGI & Qwen3-32B & $1.50\,\pm\,0.07$ & $\mathbf{1.55}\,\pm\,0.06$ & $1.46\,\pm\,0.08$ \\
 & Gemma-4-31B & $\mathbf{1.55}\,\pm\,0.11$ & $\mathbf{1.55}\,\pm\,0.11$ & $1.52\,\pm\,0.10$ \\
\addlinespace[3pt]
TaskBench & Qwen3-32B & $1.38\,\pm\,0.08$ & $\mathbf{1.47}\,\pm\,0.07$ & $1.36\,\pm\,0.08$ \\
 & Gemma-4-31B & $1.60\,\pm\,0.07$ & $\mathbf{1.68}\,\pm\,0.07$ & $1.62\,\pm\,0.06$ \\
\addlinespace[3pt]
$\tau^2$-bench & Qwen3-32B & $1.61\,\pm\,0.09$ & $\mathbf{1.63}\,\pm\,0.09$ & $1.56\,\pm\,0.08$ \\
 & Gemma-4-31B & $\mathbf{1.63}\,\pm\,0.07$ & $1.55\,\pm\,0.06$ & $1.55\,\pm\,0.06$ \\
\addlinespace[3pt]
BFCL & Qwen3-32B & $1.18\,\pm\,0.05$ & $1.28\,\pm\,0.04$ & $\mathbf{1.34}\,\pm\,0.07$ \\
 & Gemma-4-31B & $1.10\,\pm\,0.03$ & $1.12\,\pm\,0.03$ & $\mathbf{1.16}\,\pm\,0.04$ \\
\bottomrule
\end{tabular}
\end{table}

%% file: tab_main_agreement.tex
\begin{table}[htbp]
\centering\small\setlength{\tabcolsep}{5pt}
\caption{\textbf{Offline action agreement (\%), Qwen3 target.} Estimate $\pm$ one task-bootstrap SE.
Draft types used in the cited work are evaluated here on common held-out contexts.
``Direct'' disables reasoning; ``reason'' enables it. BFCL includes turn ends.
Label sampling, repeated predictions and the bootstrap protocol are detailed in Table~\ref{tab:offline}.}
\label{tab:offline_main}
\begin{tabular*}{\linewidth}{@{\extracolsep{\fill}}lcccc@{}}
\toprule
Benchmark & 0.6B trained & 0.6B base & 32B direct & 32B reason \\
 & {\scriptsize (\textbf{LEAP})} & {\scriptsize\citep{smc}}
 & {\scriptsize\citep{isp,dsp}}
 & {\scriptsize\citep{spork,swyr}} \\
\midrule
OpenAGI & $75.0\,\pm\,3.1$ & $12.8\,\pm\,1.6$ & $76.6\,\pm\,2.3$ & $97.3\,\pm\,0.9$ \\
TaskBench & $75.2\,\pm\,3.0$ & $25.0\,\pm\,1.8$ & $84.8\,\pm\,2.2$ & $96.9\,\pm\,0.9$ \\
$\tau^2$-bench & $63.7\,\pm\,4.4$ & $8.0\,\pm\,1.1$ & $42.5\,\pm\,2.6$ & $65.0\,\pm\,4.0$ \\
BFCL & $66.7\,\pm\,6.0$ & $2.0\,\pm\,1.8$ & $62.6\,\pm\,3.7$ & $91.9\,\pm\,3.2$ \\
\bottomrule
\end{tabular*}
\end{table}

%% file: tab_main_objective.tex
\begin{table}[htbp]
\centering\small
\caption{\textbf{Supervision and training objective.} Offline agreement (\%) for the Qwen3 target;
BFCL counts all decisions, including turn ends. Gold supervision also changes the training contexts.
No variant improves agreement by the three points required to trigger an end-to-end run.}
\label{tab:objective}
\begin{tabular}{llrr}
\toprule
 & & \multicolumn{2}{c}{offline agreement} \\
\cmidrule(lr){3-4}
objective & supervision & TaskBench & BFCL \\
\midrule
SFT on target actions (ours) & target's committed actions & 75.2 & 66.7 \\
SFT on gold labels & benchmark ground truth & 69.9 & 50.5 \\
GRPO from the SFT draft & exact-match reward & 77.1 & 65.7 \\
\bottomrule
\end{tabular}
\end{table}

%% file: tab_tool_latency.tex
\centering\footnotesize\setlength{\tabcolsep}{3pt}
\captionsetup{type=table}
\caption{\textbf{Tool-delay speedups.} Qwen3, trained 0.6B, $k=2$.}
\label{tab:tool}
\begin{tabular}{lrrr}
\toprule
 & Delay (s) & Pred. & Meas. \\
\midrule
$\tau^2$-bench & 0 & 1.58 & 1.58 \\
               & 2 & 1.36 & 1.41 \\
               & 5 & 1.17 & 1.20 \\
\midrule
BFCL           & 0 & 1.30 & 1.20 \\
               & 2 & 1.15 & 1.07 \\
               & 5 & 1.04 & 1.02 \\
\bottomrule
\end{tabular}

%% file: tab_online_results.tex
\begin{table}[htbp]
\centering\small\setlength{\tabcolsep}{7pt}
\caption{\textbf{Final online adapters on the test split, Qwen3 target.} Whole-task inference
speedup $\pm$ one task-bootstrap SE, using trained 0.6B drafts with matched offline and online serving settings. Training time is separate; 95\% intervals are in Table~\ref{tab:online_final_ci}.}
\label{tab:online_final}
\begin{tabular}{lccc}
\toprule
Benchmark & Offline & Online: retrain & Online: continual \\
\midrule
OpenAGI & $1.51\,\pm\,0.07$ & $1.47\,\pm\,0.08$ & $1.51\,\pm\,0.06$ \\
$\tau^2$-bench & $1.61\,\pm\,0.09$ & $1.61\,\pm\,0.10$ & $1.63\,\pm\,0.10$ \\
\bottomrule
\end{tabular}
\end{table}

%% file: sec_details.tex
\section{Additional related work}
\label{app:related}
\paragraph{Token-level speculative decoding.}
Speculative decoding~\citep{leviathan2023,chen2023} pairs a draft with a verification that keeps the target's output distribution; the draft
may be a small model or heads on the target (Medusa~\citep{medusa}, EAGLE~\citep{eagle}). A large body of
work trains the draft to raise acceptance, for example DistillSpec~\citep{distillspec} with on-policy
data and tailored divergences and domain-adapted drafts~\citep{domaindraft}. Online speculative
decoding~\citep{osd} keeps distilling a token-level draft from the queries a deployment serves; Sec.~\ref{sec:online}
does the same at the action level, where the supervision is an exact action rather than a distribution. All of it is at the token
level, where acceptance depends on matching the target's token distribution. Our setting is action-level:
only the target's decision matters, which is why we train by supervised imitation of committed actions
(Sec.~\ref{sec:train}). Our accounting of committed steps per round is the
token-level acceptance-length formula of \citet{leviathan2023} adapted to actions.

\paragraph{Action-level speculation for agents.}
Two lines exist. The first overlaps the target's \emph{decision} latency: an approximation agent runs
ahead and the target verifies. Interactive Speculative Planning~\citep{isp} and its dynamic-$k$
successor~\citep{dsp} use a cheaper prompting configuration of an API model as the approximation;
Speculative Actions~\citep{specactions} uses off-the-shelf API models and gives propositions for
single-step and asynchronous depth speculation, with latency measured on chess; Speculative Macro
Commit~\citep{smc} commits mined multi-call macros with an untrained 4B draft on three GPUs; other
systems use the same model without reasoning or off-the-shelf checkpoints~\citep{spagent,dualspec}. The second
line overlaps \emph{tool} latency: Speculative Tool Calls~\citep{stc} and PASTE~\citep{paste} execute a
predicted tool call while the model reasons, SPORK~\citep{spork} forks a probe from the running model's
own state for the same purpose, and Speculate with Memory~\citep{swm} pre-launches read-only actions from
a memory of past trajectories. AOSpec~\citep{aospec} co-speculates actions and observations so that slow
tool calls overlap generation, verifying each prediction together with the state it was made from; it
targets the regime where tools dominate, the complement of the generation-bound regime of
Sec.~\ref{sec:model}. These works bound their gain by the tool-time share (at most 2\xx{} in
\citet{stc}). LEAP reduces the number of sequential target generations and pays tool time in
full; the two compose. Speculative Actions~\citep{specactions} analyzes single-step and asynchronous
action speculation. Our framework studies synchronous proposal rounds with concurrent verification
on one shared GPU. It accounts for verification wait, segment boundaries, drafting cost and tool time
(Sec.~\ref{sec:model}). Among the self-speculation methods considered here, Speculate While You Reason~\citep{swyr},
trains the agent jointly with a self-speculation head via RL and reports Hit@1 without end-to-end
latency. LEAP studies a separately trained small action draft for a frozen target and measures wall-clock on a shared GPU.

\paragraph{Benchmarks.}
We use OpenAGI~\citep{openagi}, TaskBench~\citep{taskbench}, the $\tau^2$-bench domain of
$\tau^2$-bench~\citep{tau2bench}, and the multi-turn split of BFCL~\citep{bfcl}.

\section{Algorithm, guarantee and supervision}
\label{app:method}
\subsection{Problem, guarantee and notation}
\label{app:detail:problem}
A target agent $T$ solves a task by a sequence of actions. At step $i$ it sees a context $s_i$
(instruction, tool schemas, the actions and observations so far, and any user turns) and emits the action
$a_i = T(s_i)$, the canonical action that greedy decoding returns on $s_i$ after its reasoning; on tool benchmarks the action is
executed and its observation is appended to the context. We want to run this agent faster while every
action it executes is still its own.

Property~\ref{prop:exact} states the per-action guarantee used here.
The loop of Sec.~\ref{sec:loop} maintains this property by construction, and the harness checks it on every
task. The property holds per action, not per trajectory. The serving engine's floating-point results depend on
which requests it decodes together, and speculation verifies several contexts at once. Run alone, the target
repeats its trajectory exactly; with speculation, a small numerical difference can change a token in a long
chain of reasoning and with it a decision, after which the trajectory differs. Batching is the main source but not
the only one: with every verification decoded alone, trajectories still differ on some tasks (App.~\ref{app:success}).
Speculative Actions~\citep{specactions} calls such verification lossless; we keep the per-action guarantee and measure how
often whole trajectories differ (App.~\ref{app:success}).

\textbf{Draft and agreement.} A draft $D$ is a small model that predicts $T$'s action from the same
context, $\hat a = D(s)$, decoded greedily without reasoning. Actions are compared after
canonicalization (tool name and sorted arguments, or the reply marker for a user-facing message). On a set
of $n$ target contexts $\{(s_i, a_i)\}$ the offline \emph{agreement} of $D$ is
$\frac{1}{n}\sum_{i}\mathbbm{1}[D(s_i) = a_i]$. Its counterpart in the running system is the
per-position acceptance $q_j$ of the loop, defined in Sec.~\ref{sec:model} and compared with the offline
rate in Table~\ref{tab:agree}. Table~\ref{tab:notation} collects the notation.

\begin{table}[t]
\centering
\caption{Notation.}
\label{tab:notation}
\small
\begin{tabular}{ll}
\toprule
$T$, $D$ & target agent; draft \\
$s_i$, $a_i$ & context and target action at step $i$; $\hat a_j$: $j$-th proposal of a round \\
$k$ & proposals per round (depth) \\
$A$ & accepted proposals in a round; the round commits $A{+}1$ steps \\
$L$ & agent steps in a segment (between user turns or up to the end of the task) \\
$q_j$ & $\Pr[A \ge j]$, hit rate at position $j$; $q_1$ is the first-position acceptance \\
$c$, $e$ & committed steps per round; expected bonus-step loss at segment ends (Eq.~\ref{eq:c}) \\
$t_{\mathrm{gen}}$, $t_{\mathrm{tool}}$, $t_{\mathrm{draft}}$ & mean target generation, tool and draft-proposal times; $d = t_{\mathrm{draft}}/t_{\mathrm{gen}}$ \\
$\beta$ & round-wait inflation: mean of the slowest awaited verification over $t_{\mathrm{gen}}$ \\
$S$ & whole-task speedup, sequential over speculative wall-clock (Eq.~\ref{eq:S}) \\
\bottomrule
\end{tabular}
\end{table}

\subsection{The speculative loop}
\label{app:detail:loop}
Execution proceeds in \emph{rounds}. The draft and the target alternate; the draft never runs ahead of the
target. Given the committed prefix $s$, a round has three steps (Fig.~\ref{fig:loop}).
\begin{enumerate}
  \item \textbf{Propose.} $D$ proposes $\hat a_1,\dots,\hat a_k$ one after another, each conditioned on
  $s$ and its own earlier proposals. On tool benchmarks each proposal is executed in a copy of the
  environment so that the next proposal can see its observation; the real environment executes only
  committed actions, and an accepted proposal is not executed again: the copy in which it ran becomes the
  environment.
  \item \textbf{Verify.} $T$ generates its own action on the $k{+}1$ contexts $s$, $(s,\hat a_1)$,
  \dots, $(s,\hat a_1,\dots,\hat a_k)$, sent as concurrent requests to the same server that serves it
  alone. Let $A$ be the number of leading proposals that equal $T$'s action on their context.
  \item \textbf{Commit.} The round commits $\hat a_1,\dots,\hat a_A$ and, after them, $T$'s own action on
  the first context they do not match; if all $k$ match, that is $T$'s action on the $(k{+}1)$-st
  context, a \emph{bonus} step. The remaining verifications are discarded.
\end{enumerate}
A round thus commits $A{+}1 \le k{+}1$ steps for one target wait, and every committed action is one that
$T$ generated on the context in which it is committed (Property~\ref{prop:exact}). Rounds never cross a
\emph{segment} boundary: the end of the task, or a reply to the user, after which the next context
contains a user turn the draft cannot predict. When a proposal is a reply, no request is issued past it,
and a round that verifies the reply ends without a bonus step. $L$ denotes the number of agent steps in a
segment. Fig.~\ref{fig:timelines} shows where the time goes when every proposal is accepted and when
the first one is rejected.

\begin{proposition}[commit-eligible prefixes are teacher-forced]
\label{prop:commit}
$\hat a_j$ is committed only if $\hat a_{<j} = a_{<j}$; hence every context in which a proposal can be
committed is a prefix of the target's own trajectory, at every position $j$ of the round and not only at
$j{=}1$, and the committed action in that context is the target's teacher-forced action. Deployment
therefore draws its (context, committed action) pairs from the target's own traces; which contexts are
drawn depends on where rounds start and which proposals are accepted, a selection that
Sec.~\ref{sec:validation} measures.
\end{proposition}
\emph{Proof.} A round commits $\hat a_1,\dots,\hat a_A$ with $A = \min\{j : \hat a_j \neq a_j\} - 1$. If
$\hat a_j$ is committed then $j \le A$, so $\hat a_{<j} = a_{<j}$ and $\hat a_j$ was produced by
$D$ on $(s, a_{<j})$. Proposals made on a mismatched prefix are discarded, never committed.
\hfill$\square$

If the target's output depended only on the context, and not on the requests decoded with it or before
it, and if tool results depended only on the environment state and the call, then by
Proposition~\ref{prop:commit} and induction over rounds the committed trajectory would be exactly the
target-only one. Property~\ref{prop:exact} needs neither condition; the first does not hold exactly in our
serving engine (Sec.~\ref{sec:problem} and App.~\ref{app:success}).

\paragraph{Implementation.}
Verification responses are consumed in prefix order; as soon as context $j$ mismatches, the requests for
contexts $j{+}1,\dots,k$ are cancelled and the server aborts them, so a round never waits for a
verification it will discard. Both models are served by the same engine on one GPU; the draft is a LoRA
adapter on a 0.6B base run without reasoning (App.~\ref{app:serving}).

\begin{figure}[tbp]
\centering\setlength{\fboxsep}{2pt}
\renewcommand{\looptagnode}[1]{}\renewcommand{\looptagA}[1]{}
\resizebox{0.9\linewidth}{!}{\input{figures/fig_loop_legend}}\par\smallskip
\begin{subfigure}{\linewidth}\centering
\resizebox{\linewidth}{!}{\input{figures/fig_loop_accept}}
\caption{Both proposals accepted: three committed steps in one verification round.}
\end{subfigure}\par\smallskip
\begin{subfigure}{\linewidth}\centering
\resizebox{\linewidth}{!}{\input{figures/fig_loop_reject}}
\caption{First proposal rejected: one committed step, plus proposing cost.}
\end{subfigure}
\caption{\textbf{Round timelines with $k=2$.} Illustrative durations. A successful round waits for the
slowest required verification; on a mismatch, unused requests are cancelled. Tool execution for
proposals occurs while proposing, and accepted environment states are adopted.}
\label{fig:timelines}
\end{figure}
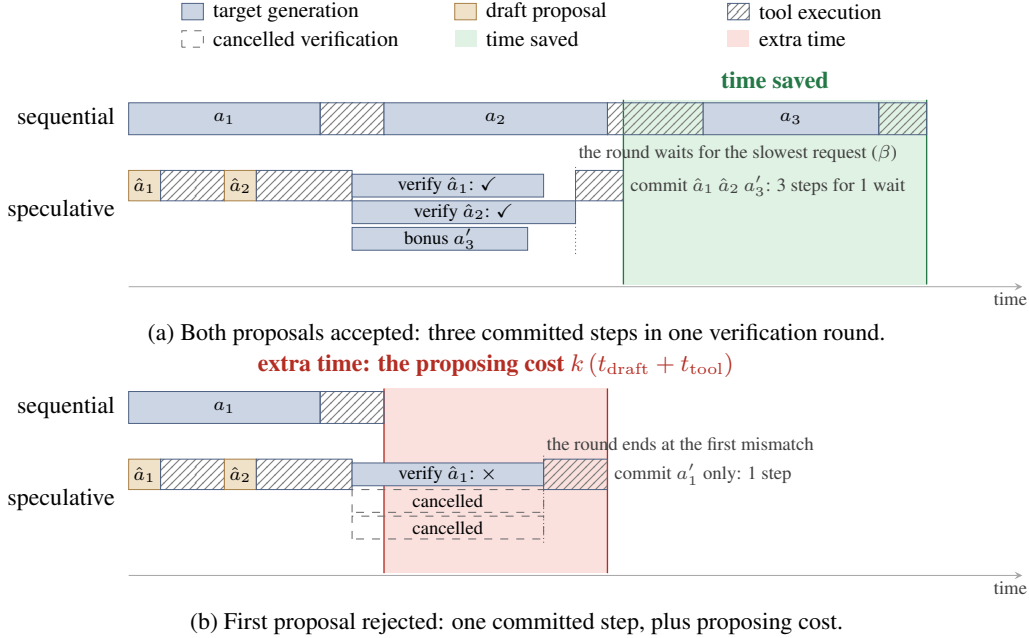

\subsection{Training the draft on the target's committed actions}
\label{app:detail:train}
\textbf{Supervision.} The training set is a by-product of running the agent (Fig.~\ref{fig:loop}): every step yields a pair
$(s_i, a_i)$ of context and the target's committed action. We run $T$ once over the training tasks, log
these pairs, and fine-tune $D$ by cross-entropy on the action tokens given the context. No target logits,
no draft rollouts, and no human labels are used. The separate pass is a convenience, not a requirement:
once the speculative system runs, every action it commits is one the target generated
(Proposition~\ref{prop:commit}), whatever the draft proposed, so its commits are the same supervision.
Sec.~\ref{sec:online} trains the draft from them while the system serves.

\textbf{Predict the decision.} The label $a_i$ is the canonical action, not the target's reasoning text.
The draft is asked to reproduce a low-entropy policy's next decision, not a 32B model's output
distribution, which is what makes a 0.6B base viable. The draft runs with reasoning disabled and, on
tool benchmarks, under the tool schema as a decoding constraint. Implementation: LoRA rank 64,
$\alpha{=}128$, learning rate $3\times10^{-5}$, two to three epochs with the checkpoint chosen on a
validation split; training takes minutes on one GPU per benchmark.

\paragraph{Why imitation is a natural objective.}
By Proposition~\ref{prop:commit} (Sec.~\ref{sec:loop}), a proposal is committed only if it was generated
on the target's true prefix, so every context in which a proposal can be committed is a prefix of the
target's own trajectory, at every position and not only the first. Supervised imitation on the target's
traces trains $D$ on contexts of this kind. Acceptance depends only on the target's decision, so the label
is the canonical action, not the target's full output distribution. Plain cross-entropy on the canonical
action is therefore the natural and cheapest objective. It is a surrogate for exact-match acceptance, and
Sec.~\ref{sec:agreement} tests alternatives. The multi-step figure of merit is the per-position hit-rate
curve $q_j$ of Sec.~\ref{sec:model}, not the single-step agreement.

\section{Latency framework and additional checks}
\label{app:model}
\input{sec_framework}
\section{Complete experimental results}
\label{app:results}
\label{app:detail:experiments}
\subsection{Setup}
\label{app:detail:setup}
We use two targets. The main one is Qwen3-32B (temperature 0, reasoning enabled); the second,
Gemma-4-31B-it, repeats the whole pipeline in Sec.~\ref{sec:results4}. The drafts are
Qwen3-0.6B, 1.7B and 4B, each with a LoRA adapter trained as in Sec.~\ref{sec:train} on the traces of
one target and one benchmark; the 0.6B draft is the main one, and the offline rates in this section are
for the Qwen3 target. Target and draft are served on one H100; the draft is co-located with the target
and takes about a tenth of the GPU memory. Splits are task-disjoint.
OpenAGI~\citep{openagi}: 250 train / 63 test tasks, 1{,}090 training pairs, ReAct-style planning with a
fixed tool set. TaskBench~\citep{taskbench}: 228 train / 60 test tasks, 1{,}120 pairs, strict canonical
tool schemas; no test task's tool sequence appears in training. $\tau^2$-bench~\citep{tau2bench}: we use its telecom domain only, in solo mode (the agent works without a
simulated user); 60 / 14 / 40 train / validation / test tasks, 932 pairs, 44 tools executed against the
benchmark's local backend. BFCL is the multi-turn split of the Berkeley Function
Calling Leaderboard~\citep{bfcl}: 130 train / 60 test tasks, 1{,}254 pairs; user
turns are speculation boundaries and proposals execute in cloned environments. Each test task is run once
sequentially (target only) and once speculatively on the same node and service, and we report the pooled
ratio $\sum \text{seq} / \sum \text{spec}$ of whole-task wall-clock (App.~\ref{app:serving} gives the timing protocol).

\begin{table}[t]
\centering
\caption{\textbf{Offline action agreement with the Qwen3-32B target} (exact match on the target's own
held-out contexts; repeated predictions averaged per context, except the $\tau^2$-bench
32B-reason reference, which uses the mode of three passes). ``32B-direct'' is the target with
reasoning disabled; ``32B-reason'' is the target itself with reasoning on, replayed on the same contexts,
i.e. on the planning benchmarks its agreement with its own earlier pass, a reference set by the engine's nondeterminism rather than a bound. The three prior-work columns are an off-the-shelf small model, a cheaper
configuration of the same model, and the target as its own draft, for which ``32B-reason'' is the reference, all
measured in one setting. BFCL rows are on 99 development contexts and count every decision, a turn end or an exact tool group;
the tool-group rate used for the acceptance comparison is in Table~\ref{tab:agree}. $\tau^2$-bench labels were sampled from the target at temperature 0.6,
and all $\tau^2$-bench columns are measured against them; there, as on BFCL, the 32B-reason column replays the
target (reasoning on, temperature 0) on the same contexts.}
\label{tab:offline}
\small
\resizebox{\linewidth}{!}{%
\begin{tabular}{lrrrrr}
\toprule
 & & \multicolumn{3}{c}{drafts used in prior work} & ours \\
\cmidrule(lr){3-5}\cmidrule(lr){6-6}
benchmark & contexts & untrained 0.6B & 32B-direct & 32B-reason & trained 0.6B \\
 & & {\scriptsize\citep{smc}} & {\scriptsize\citep{isp,dsp}} & {\scriptsize\citep{spork,swyr}} & \\
\midrule
OpenAGI   & 263 & 12.8 & 76.6 & 97.3 & 75.0 \\
TaskBench & 277 & 25.0 & 84.8 & 96.9 & 75.2 \\
$\tau^2$-bench   & 226 & 8.0  & 42.5 & 65.0 & 63.7 \\
BFCL      & 99  & 2.0  & 62.6 & 91.9 & 66.7 \\
\bottomrule
\end{tabular}}
\end{table}

The standard errors in Table~\ref{tab:offline_main} use 10,000 task-level bootstrap resamples.
All contexts and repeated predictions from a task stay together; each replicate recomputes the
context-weighted agreement. The evaluated sets contain 63 OpenAGI, 57 TaskBench, 14 $\tau^2$-bench
and 10 BFCL tasks. These SEs quantify task-sampling uncertainty conditional on the evaluated models
and runs, rather than variation across training seeds.

Table~\ref{tab:offline} gives the offline agreement of the drafts used below. Its columns are the draft
families of prior work, measured here in one setting: an untrained small model, the target with reasoning
disabled (a cheaper configuration of the same model), and, as a reference for self-speculation, the target
replaying itself with reasoning on (32B-reason). Training raises the 0.6B draft from 2--25\% to 64--75\% on every benchmark. On $\tau^2$-bench, whose labels were sampled at temperature 0.6, even the target replayed at temperature 0 matches only 65.0\% of them, so the trained draft (63.7\%) is 1.3 points below the target itself, and the 1.7B and 4B drafts of Sec.~\ref{sec:training} exceed it (68.1\% and 69.5\%). The target with reasoning disabled is a strong
predictor of its own decisions on the planning benchmarks (76.6 and 84.8, at or above the trained draft)
but not on the tool benchmarks: on $\tau^2$-bench it agrees with its reasoning-on decisions on 42.5\% of
contexts, 21 points below the trained 0.6B draft, and on BFCL on 62.6\%, four points below. Where the
decision depends on reasoning over a policy and a conversation, disabling reasoning changes the decision,
and a small draft trained on the reasoning-on decisions predicts them better than the target itself does
without reasoning.

\begin{table}[t]
\centering
\caption{\textbf{From offline agreement to live acceptance} for the trained 0.6B drafts. Offline:
exact match on the target's held-out contexts (Table~\ref{tab:offline}: mean over repeated predictions; BFCL: tool-group rate). Live,
from the end-to-end runs of Table~\ref{tab:four}: $q_1$, the fraction of rounds whose first proposal is
accepted; $q_2/q_1$ and $q_3/q_2$, the ratios of successive hit rates, which also count the rounds that
stop proposing before the position; $a_2$ and $a_3$, the acceptance of the second and third proposal
among the rounds that made it and accepted the earlier ones; and the pooled per-proposal acceptance.
BFCL live rates are per tool group. A dash marks a position that does not exist: $\tau^2$-bench and BFCL run
with $k{=}2$.}
\label{tab:agree}
\small
\begin{tabular}{lrrrrrrr}
\toprule
benchmark & offline & $q_1$ & $q_2/q_1$ & $q_3/q_2$ & $a_2$ & $a_3$ & per proposal \\
\midrule
OpenAGI   & 75.0 & 0.71 & 0.63 & 0.63 & 0.81 & 0.80 & 0.75 \\
TaskBench & 75.2 & 0.66 & 0.51 & 0.53 & 0.66 & 0.73 & 0.68 \\
$\tau^2$-bench   & 63.7 & 0.64 & 0.54 & --   & 0.63 & --   & 0.64 \\
BFCL      & 49.2 & 0.38 & 0.20 & --   & 0.62 & --   & 0.53 \\
\bottomrule
\end{tabular}
\end{table}

\paragraph{From agreement to acceptance.}
Table~\ref{tab:agree} puts the offline and the live rates side by side. At the first position, live
acceptance is within four points of offline agreement on OpenAGI and $\tau^2$-bench and nine points below it on
TaskBench, where the speculative run's trajectories depart from the canonical ones on more tasks
(trajectory parity 0.86 against 0.98 on OpenAGI), so the draft is scored on contexts it has not seen in
that form. On BFCL a quarter of the rounds are hand-off probes at turn boundaries that carry no proposal,
which lowers $q_1$ below the per-group rate. The hit rates fall with position ($q_2/q_1$ is 0.51--0.63),
but not because later proposals are harder: among the rounds that make a second or third proposal and
accepted the earlier ones, the acceptance ($a_2$, $a_3$) is 0.81 and 0.80 on OpenAGI against 0.71 at the
first position, and flat on TaskBench and $\tau^2$-bench. The fall comes from rounds that run into the end of
the task, or into an invalid draft output, so that the later positions do not exist: on OpenAGI the
second and third positions exist in 81\% and 62\% of the rounds. This is why the framework of
Sec.~\ref{sec:model} carries per-position rates $q_j$, which include this truncation, rather than a
single acceptance rate.

\input{sec_results_v2}

\section{Additional training studies}
\label{app:training}
\input{sec_training}
\section{Additional limitations}
\label{app:limitations}
The speedups are 1.2--1.6\xx{}, and the framework of Sec.~\ref{sec:model} explains why they are not larger on
these benchmarks: tasks are short (4--14 steps), BFCL's user turns cut segments to 2.6 steps, and the
round cost $\beta$ of the trained drafts is 1.1--1.3. We evaluate two
target families (Qwen3-32B and Gemma-4-31B) on one GPU class, and our tools run against local simulated backends. With added tool latency the speedup
falls, to 1.20\xx{} on $\tau^2$-bench and 1.02\xx{} on BFCL at 5\,s per tool step (Table~\ref{tab:tool}), because the
draft's proposals execute their tools while proposing; with slow tools the tool-overlap
mechanisms of \citet{stc,swm,aospec} would matter more than ours. Every action LEAP executes is
target-authored, but the trajectories are not those of a target-only run: speculation changes which requests the
engine decodes together, and on the tool benchmarks, whose steps reason for 10--12\,s, most speculative
trajectories take a different decision at some step (App.~\ref{app:success}; two $\tau^2$-bench tasks even ran to
the token budget in the speculative arm only). Task success shows no systematic change. Proposing and verifying do not overlap, so the draft's cost is paid in full; a draft that ran ahead of the
target would hide that term, which Eq.~\ref{eq:S} without the draft term puts at 1.47$\to$1.72\xx{} on OpenAGI
and 1.46$\to$1.66\xx{} on TaskBench and at a few hundredths on the tool benchmarks. Verification issues up to
$k{+}1$ target generations per round (1.16\xx{} the sequential completion tokens on $\tau^2$-bench), so the
method trades compute for latency. Test sets are 40--63 tasks; task-level bootstrap intervals of the pooled
ratio are 0.15--0.30 wide (Table~\ref{tab:four}), so differences of a few
hundredths between configurations are not resolved by these runs. The online experiment of
Sec.~\ref{sec:online} makes one pass over each benchmark's training split; retraining on the whole buffer
costs more with every batch, and the continual variant that fixes this cost dipped once on OpenAGI before recovering,
which we have not studied further.

%% file: figures/fig_loop_legend.tex
\begin{tikzpicture}[x=1.05cm, y=1cm, font=\footnotesize]
\useasboundingbox (-1.75,-0.1) rectangle (11.4,0.78);
\foreach \x/\y/\sty/\txt in {0.2/0.45/gen/target generation, 4.0/0.45/draftbar/draft proposal, 7.8/0.45/toolbar/tool execution,
                              0.2/0.05/cancelbar/cancelled verification, 4.0/0.05/savedband/time saved, 7.8/0.05/extraband/extra time}{
  \draw[\sty] (\x,\y+0.03) rectangle (\x+0.3,\y+0.27); \node[anchor=west] at (\x+0.32,\y+0.15) {\txt};}
\end{tikzpicture}

%% file: figures/fig_loop_accept.tex
\begin{tikzpicture}[x=1.05cm, y=1cm, font=\scriptsize, >=stealth]
\loopframe{(c) both proposals accepted}
\loopband{6.2}{10.0}{savedband}{savedline}{time saved}
\node[rowlab] at (-0.1,2.16) {sequential};
\loopbar{0}{2.4}{1.95}{gen}{0.42}{$a_1$}\loopbar{2.4}{3.2}{1.95}{toolbar}{0.42}{}
\loopbar{3.2}{6.0}{1.95}{gen}{0.42}{$a_2$}\loopbar{6.0}{7.2}{1.95}{toolbar}{0.42}{}
\loopbar{7.2}{9.4}{1.95}{gen}{0.42}{$a_3$}\loopbar{9.4}{10.0}{1.95}{toolbar}{0.42}{}
\node[rowlab] at (-0.1,0.95) {speculative};
\loopbar{0}{0.4}{1.08}{draftbar}{0.40}{$\hat a_1$}\loopbar{0.4}{1.2}{1.08}{toolbar}{0.40}{}
\loopbar{1.2}{1.6}{1.08}{draftbar}{0.40}{$\hat a_2$}\loopbar{1.6}{2.8}{1.08}{toolbar}{0.40}{}
\loopbar{2.8}{5.2}{1.13}{gen}{0.30}{verify $\hat a_1$: \checkmark}
\loopbar{2.8}{5.6}{0.78}{gen}{0.30}{verify $\hat a_2$: \checkmark}
\loopbar{2.8}{5.0}{0.43}{gen}{0.30}{bonus $a'_3$}
\draw[densely dotted, black!60] (5.6,0.38) -- (5.6,1.55);
\node[note, anchor=south west] at (5.6,1.55) {the round waits for the slowest request ($\beta$)};
\loopbar{5.6}{6.2}{1.08}{toolbar}{0.40}{}
\node[note, anchor=west] at (6.25,1.28) {commit $\hat a_1\,\hat a_2\,a'_3$: 3 steps for 1 wait};
\end{tikzpicture}

%% file: figures/fig_loop_reject.tex
\begin{tikzpicture}[x=1.05cm, y=1cm, font=\scriptsize, >=stealth]
\loopframe{(d) first proposal rejected}
\loopband{3.2}{6.0}{extraband}{extraline}{extra time: the proposing cost $k\,(t_{\mathrm{draft}}+t_{\mathrm{tool}})$}
\node[rowlab] at (-0.1,2.16) {sequential};
\loopbar{0}{2.4}{1.95}{gen}{0.42}{$a_1$}\loopbar{2.4}{3.2}{1.95}{toolbar}{0.42}{}
\node[rowlab] at (-0.1,0.95) {speculative};
\loopbar{0}{0.4}{1.08}{draftbar}{0.40}{$\hat a_1$}\loopbar{0.4}{1.2}{1.08}{toolbar}{0.40}{}
\loopbar{1.2}{1.6}{1.08}{draftbar}{0.40}{$\hat a_2$}\loopbar{1.6}{2.8}{1.08}{toolbar}{0.40}{}
\loopbar{2.8}{5.2}{1.13}{gen}{0.30}{verify $\hat a_1$: $\times$}
\loopbar{2.8}{5.2}{0.78}{cancelbar}{0.30}{cancelled}
\loopbar{2.8}{5.2}{0.43}{cancelbar}{0.30}{cancelled}
\draw[densely dotted, black!60] (5.2,0.38) -- (5.2,1.55);
\node[note, anchor=south west] at (5.2,1.55) {the round ends at the first mismatch};
\loopbar{5.2}{6.0}{1.08}{toolbar}{0.40}{}
\node[note, anchor=west] at (6.05,1.28) {commit $a'_1$ only: 1 step};
\end{tikzpicture}

%% file: sec_framework.tex
\subsection{A latency framework for the loop}
\label{app:detail:model}

We give a latency framework for the rounds of Sec.~\ref{sec:loop}. It
separates the quantities that decide the speedup (acceptance, depth, task
structure, round cost, draft cost), and it predicts the end-to-end speedup of
a configuration from its hit rates, its draft cost and a sequential trace
(Sec.~\ref{sec:validation}), which also tests predictions using offline hit rates.

The framework covers exact-match action verification on one shared GPU. Proposing finishes
before verification begins. It does not model asynchronous draft run-ahead, semantic verification,
or mechanisms that overlap tool execution with model generation.

\paragraph{Quantities.}
For a benchmark and a target we measure the mean wall-clock of one sequential
target step $t_{\mathrm{gen}}$ (generation only), of one tool execution
$t_{\mathrm{tool}}$, and of one draft proposal $t_{\mathrm{draft}}$, and write
$d = t_{\mathrm{draft}}/t_{\mathrm{gen}}$. The \emph{hit rates}
$q_j = \Pr[A \ge j]$, $j = 1,\dots,k$, are the probabilities that the first $j$
proposals of a round all match; $q_1$ is the first-position acceptance. The
\emph{round-wait inflation} $\beta$ accounts for the round waiting for the
slowest of the $A{+}1$ verifications it uses: its generation cost is the
expected maximum of those requests rather than one mean step, written
$\beta\, t_{\mathrm{gen}}$. If the awaited step times were
independent draws from the sequential step-time distribution,
$\beta = \sum_{a} \Pr[A{=}a]\; \mathbb{E}[\max_{a+1}] / t_{\mathrm{gen}} \ge 1$, a
closed form in the empirical distribution; consecutive steps are correlated,
so we also replay the actual trace (Sec.~\ref{sec:validation}). We compare both estimates with the measured value (Fig.~\ref{fig:components}(b)).

\paragraph{Commits per round.}
A round commits $A{+}1$ steps unless the segment ends inside the accepted
prefix, in which case the bonus step is lost. Writing $e$ for the expected loss
per round,
\begin{equation}
  c \;=\; 1 + \sum_{j=1}^{k} q_j \;-\; e,
  \qquad e \;=\; \Pr[\text{segment ends within the accepted prefix}].
  \label{eq:c}
\end{equation}
If hits were independent with rate $p$ and segments were long ($e \to 0$),
Eq.~\ref{eq:c} would be the token-level expected acceptance length of
\citet{leviathan2023}, $c = 1 + p + \dots + p^{k}$. Measured hit rates decay
faster than $q_1^{\,j}$ (Fig.~\ref{fig:components}(a)), partly because rounds
reach task or user-turn boundaries. Acceptance conditional on an existing proposal and an accepted
prefix need not decline (Table~\ref{tab:agree}). On OpenAGI the second and third positions exist in
81\% and 62\% of the rounds. For the same reason $e$ is not small on short
tasks (0.38 on OpenAGI, whose tasks average 4.2 steps). In all cases
$c \le \min(k{+}1, L)$: short segments cap the achievable depth regardless of
acceptance.

\paragraph{Round time and speedup.}
Let $D$ be the mean proposing time per round, including proposal-tool execution, divided by
$t_{\mathrm{gen}}$. Let $f$ be the fraction of sequential steps that execute a tool and $g$ the
mean number of target-executed tool actions per round. Write $\tau=t_{\mathrm{tool}}/t_{\mathrm{gen}}$.
A sequential step costs $t_{\mathrm{gen}}(1+f\tau)$. A speculative round costs approximately
$t_{\mathrm{gen}}(D+\beta+g\tau)$. Accepted proposal states are adopted, so their tool costs are
included in $D$ only once. For comparable work in the two arms, the speedup is
\begin{equation}
 S\simeq\frac{c(1+f\tau)}{\beta+D+g\tau}
 \quad\xrightarrow{\;\tau\to0\;}\quad\frac{c}{\beta+D}.
 \label{eq:S}
\end{equation}
This is the same model as Eq.~\ref{eq:main:S}. If every round makes exactly $k$ proposals and
tool time is negligible, then $D=kd$. More generally, if each proposal executes a tool and
$f=g=1$, then $D=k(d+\tau)$ and the denominator is $\beta+k(d+\tau)+\tau$.

The framework separates four effects. Acceptance controls $c$ and the number of awaited
verifications. Depth can increase $c$, proposing cost and verification wait. Segment boundaries
limit both the number of proposals and the bonus action. The draft's relative cost enters through
$D$, which explains why a smaller draft can be faster at comparable acceptance
(Sec.~\ref{sec:selfspec}). With negligible tool time, two bounds follow:
$S\le c/\beta\le(k+1)/\beta$. For long segments with independent hits, the acceptance ceiling
approaches $1/((1-p)\beta)$ as $k$ grows. When all $k$ proposals are made, their cost $kd$ keeps
growing, so deeper speculation eventually stops paying off.

Verification issues up to $k+1$ target requests per round. The number of requests per committed
step is therefore at most $(k+1)/c$. Discarded requests are cancelled early; the measured completion-token
ratio is 1.16 on $\tau^2$-bench. These costs apply to our synchronous propose--verify schedule.
Speculative Actions~\citep{specactions} overlaps draft and target execution, so its latency analysis
uses a different execution schedule.

\paragraph{Interpreting acceptance.}
For long segments, independent hits of probability $p$, negligible tool time and exactly $k$
proposals per round, the model becomes
\begin{equation}
 S(p)\simeq\frac{1+p+\cdots+p^k}{\beta(p)+kd}.
 \label{eq:main:acceptance}
\end{equation}
The numerator is nonlinear for $k\ge2$, and verification wait can vary with $p$. Boundaries add
another dependence: a one-action segment commits only one action even with perfect agreement.
Thus equal acceptance need not imply equal speedup.

The measured controls in Fig.~\ref{fig:curve} have different accepted-prefix distributions,
segment truncation and proposing costs. Pooled acceptance alone does not identify those quantities,
so Eq.~\ref{eq:S} does not imply a universal speedup curve in that acceptance rate.
For example, the untrained TaskBench draft produces no usable proposal in 38\% of rounds. Those
rounds consume proposing time and commit a target action, although no proposal enters the
accepted/compared ratio. The framework therefore uses per-position rates, boundary loss and
round costs explicitly. Tables~\ref{tab:model} and~\ref{tab:four} retain the measured inputs;
the size and online studies are shown separately in Figs.~\ref{fig:size} and~\ref{fig:online}.

\paragraph{Choosing $k$.}
For negligible tool time, write $\pi_k$ for the fraction of rounds that make a $k$-th proposal.
The proposing cost is $D_k\simeq d\sum_{j\le k}\pi_j$; proposals can stop at segment boundaries.
Write $\Delta\beta_k$ for the growth of
$\beta$ from awaiting one more request. With $B_k = \beta_k + d \sum_{j \le k} \pi_j$ the round cost at
depth $k$, so that $S_k = c_k / B_k$, a $k$-th proposal pays off if and only if
\begin{equation}
 c_k-c_{k-1}>S_{k-1}(B_k-B_{k-1}).
 \label{eq:main:depth}
\end{equation}
By Eq.~\ref{eq:c}, the gain $c_k - c_{k-1}$ is $q_k$ minus the
growth of $e$, plus any change in the earlier hit rates, which move because depth changes where rounds
start. If end effects are negligible and increasing depth leaves the earlier hit rates and proposal probabilities unchanged, the condition simplifies to
\begin{equation}
  q_k \;>\; S_{k-1}\,\bigl(\pi_k\, d + \Delta\beta_k\bigr),
  \qquad\text{or, with } \pi_k = 1 \text{ and } \Delta\beta_k = 0,\qquad
  q_k \;>\; S_{k-1}\, d .
  \label{eq:k}
\end{equation}
\begin{corollary}[depth with stable earlier-position rates]
\label{cor:k}
Assume negligible end effects and unchanged earlier-position hit rates and proposal probabilities when increasing depth. Adding a $k$-th proposal pays off when its hit probability exceeds the
current speedup times the marginal cost of the proposal: the draft's relative
cost, paid only in the rounds that reach the position, plus the extra round-wait
inflation. Near the end of a segment an accepted extra proposal can replace the bonus step instead of
adding one, so on short tasks the rule overstates the value of depth. No proposal is made past the segment
end ($\pi_k = 0$ for $k > L$); when $\Delta\beta_k=0$, setting $\pi_k=1$ upper-bounds the marginal drafting cost; this need not be conservative if verification wait also increases.
\end{corollary}
Using this approximation as a depth heuristic requires offline hit rates, draft latency, and an estimate of the change in verification wait. The following examples illustrate the heuristic rather than establish reliable configuration ranking. The planning hit rates are from the $k=5$ runs; the speedups compare separate depth runs. On OpenAGI
($d = 0.06$) the simplified threshold is about $0.09$ and selects $k = 4$
($q_3 = 0.31$, $q_4 = 0.14$, $q_5 = 0.08$); with $\pi_4 = 0.41$, $\pi_5 = 0.15$
and $\Delta\beta \approx 0.015$ per position the thresholds fall to $0.06$ and
$0.04$, and the rule predicts gains of a few hundredths up to $k = 5$; the measured speedup rises from
1.34\xx{} at $k{=}1$ to 1.51\xx{} at $k{=}3$ and 1.57\xx{} at $k{=}5$ (Table~\ref{tab:model}), the last
step within the bootstrap interval. On TaskBench ($d = 0.07$, threshold about $0.10$; $q_2 = 0.37$,
$q_3 = 0.22$, $q_4 = 0.16$, $q_5 = 0.06$) it selects $k = 4$, and the measured speedup rises
from 1.27\xx{} at $k{=}1$ to 1.37\xx{} at $k{=}3$ and 1.41\xx{} at $k{=}5$
(Table~\ref{tab:model}), the last step again within noise. On BFCL (threshold $0.02$, $q_2 = 0.07$)
it selects $k = 2$, and the measured speedup is 1.15, 1.20 and 1.19\xx{} at $k = 1$, 2 and 3; on $\tau^2$-bench (threshold $0.02$, $q_2 = 0.35$, $q_3 = 0.27$) it predicts that $k = 3$ gains further,
and the measured speedup rises from 1.58\xx{} at $k{=}2$ to 1.79\xx{} at $k{=}3$
(Table~\ref{tab:model}).

\paragraph{Assumptions.}
The framework assumes that (a) both arms do the same work. Every committed action is
target-authored, but a speculative trajectory can differ from the target-only one
(App.~\ref{app:success}); per run, the speculative arm commits between 3.9\% fewer and
2.5\% more steps than the sequential arm in the configurations of Table~\ref{tab:model};
(b) concurrent verification on a dedicated GPU is close to free: on the Qwen3 target we measure
$+1$ to $+4\%$ per generated token with up to four concurrent requests and
unchanged generation length, so $\beta$ is a max-of-$n$ effect rather than a
throughput effect (on the Gemma planning configurations of Table~\ref{tab:four}, measured $\beta$
exceeds the offline-hit replay by 0.05--0.13, Sec.~\ref{sec:validation}); (c) tool time is small relative to generation on the
benchmarks we use (simulated backends, $\le 0.07$\,s per step); we exercise the
$t_{\mathrm{tool}}$ terms by adding 2 or 5\,s to every executed tool step on $\tau^2$-bench and BFCL
(Table~\ref{tab:tool}); (d) proposing and
verifying do not overlap, as in Sec.~\ref{sec:loop}; a draft that ran ahead of
the target could hide part of the proposing cost $D$, which we leave to future work; (e) the
harness cancels discarded verifications (Sec.~\ref{sec:loop}), and we report the
residual idle time separately.

%% file: sec_results_v2.tex
\subsection{Speedup across four benchmarks}
\label{app:detail:results4}

\begin{table}[t]
\centering
\caption{\textbf{Whole-task end-to-end speedup of LEAP: four benchmarks, three draft sizes, two targets} (one H100
shared by target and draft; $\tau^2$-bench and BFCL execute tools; 63, 60, 40 and 60 test tasks). Each benchmark
serves its three drafts in one configuration (App.~\ref{app:serving}). Agreement is the draft's offline exact-match rate (\%) on the
target's held-out contexts (BFCL: decision level); acceptance is the share of verified proposals that were
accepted in the runs (the per-proposal rate of Table~\ref{tab:agree}); $c$ is the measured number of committed steps per round (Eq.~\ref{eq:c}); the 95\% interval is a
task-level bootstrap of the pooled ratio (10{,}000 resamples). Bold marks the fastest of the three drafts for
a benchmark and target. Task success per arm shows no systematic change (App.~\ref{app:success}).}
\label{tab:four}
\footnotesize\setlength{\tabcolsep}{4pt}\renewcommand{\arraystretch}{0.92}
\begin{tabular}{llrrrrrrrc}
\toprule
benchmark & draft & steps & $t_{\mathrm{gen}}$ (s) & $k$ & agreem. & accept. & $c$ & speedup & 95\% CI \\
\midrule
\multicolumn{10}{l}{\emph{Target Qwen3-32B}} \\
\midrule
\multirow{3}{*}{OpenAGI} & 0.6B & 4.2  & 1.1  & 3 & 75.4 & 0.75 & 2.05 & 1.50\xx{} & [1.36, 1.65] \\
                & 1.7B & 4.2  & 1.1  & 3 & 77.6 & 0.78 & 2.18 & \textbf{1.55\xx{}} & [1.43, 1.68] \\
                & 4B   & 4.2  & 1.1  & 3 & 78.7 & 0.79 & 2.26 & 1.46\xx{} & [1.30, 1.62] \\
\cmidrule(lr){2-10}
\multirow{3}{*}{TaskBench} & 0.6B & 5.0  & 1.8  & 3 & 74.7 & 0.70 & 1.94 & 1.38\xx{} & [1.24, 1.55] \\
                & 1.7B & 5.0  & 1.8  & 3 & 81.0 & 0.76 & 2.07 & \textbf{1.47\xx{}} & [1.34, 1.61] \\
                & 4B   & 5.0  & 1.8  & 3 & 80.5 & 0.75 & 2.11 & 1.36\xx{} & [1.22, 1.52] \\
\cmidrule(lr){2-10}
\multirow{3}{*}{$\tau^2$-bench} & 0.6B & 13.9 & 12.4 & 2 & 63.3 & 0.65 & 1.92 & 1.61\xx{} & [1.44, 1.79] \\
                & 1.7B & 13.9 & 12.5 & 2 & 68.1 & 0.63 & 1.88 & \textbf{1.63\xx{}} & [1.46, 1.81] \\
                & 4B   & 13.9 & 12.4 & 2 & 69.5 & 0.66 & 1.90 & 1.56\xx{} & [1.43, 1.73] \\
\cmidrule(lr){2-10}
\multirow{3}{*}{BFCL multi-turn} & 0.6B & 10.1 & 10.8 & 2 & 66.7 & 0.53 & 1.44 & 1.18\xx{} & [1.08, 1.28] \\
                & 1.7B & 10.1 & 10.8 & 2 & 72.7 & 0.59 & 1.55 & 1.28\xx{} & [1.20, 1.36] \\
                & 4B   & 10.1 & 10.8 & 2 & 79.8 & 0.65 & 1.67 & \textbf{1.34\xx{}} & [1.21, 1.48] \\
\midrule
\multicolumn{10}{l}{\emph{Target Gemma-4-31B-it}} \\
\midrule
\multirow{3}{*}{OpenAGI} & 0.6B & 4.3  & 2.1 & 3 & 79.8 & 0.78 & 2.18 & \textbf{1.55\xx{}} & [1.35, 1.79] \\
                & 1.7B & 4.3  & 2.1 & 3 & 77.3 & 0.75 & 2.08 & \textbf{1.55\xx{}} & [1.35, 1.78] \\
                & 4B   & 4.3  & 2.1 & 3 & 79.9 & 0.76 & 2.15 & 1.52\xx{} & [1.34, 1.73] \\
\cmidrule(lr){2-10}
\multirow{3}{*}{TaskBench} & 0.6B & 4.9  & 2.3 & 3 & 80.1 & 0.82 & 2.30 & 1.60\xx{} & [1.47, 1.76] \\
                & 1.7B & 4.9  & 2.3 & 3 & 84.6 & 0.83 & 2.43 & \textbf{1.68\xx{}} & [1.56, 1.82] \\
                & 4B   & 4.9  & 2.3 & 3 & 84.2 & 0.85 & 2.45 & 1.62\xx{} & [1.51, 1.75] \\
\cmidrule(lr){2-10}
\multirow{3}{*}{$\tau^2$-bench} & 0.6B & 14.7 & 4.7 & 2 & 72.8 & 0.80 & 2.27 & \textbf{1.63\xx{}} & [1.50, 1.77] \\
                & 1.7B & 14.7 & 4.7 & 2 & 76.4 & 0.79 & 2.20 & 1.55\xx{} & [1.44, 1.67] \\
                & 4B   & 14.7 & 4.7 & 2 & 77.6 & 0.80 & 2.24 & 1.55\xx{} & [1.43, 1.67] \\
\cmidrule(lr){2-10}
\multirow{3}{*}{BFCL multi-turn} & 0.6B & 10.2 & 3.3 & 2 & 63.6 & 0.50 & 1.41 & 1.10\xx{} & [1.04, 1.16] \\
                & 1.7B & 10.2 & 3.3 & 2 & 69.7 & 0.55 & 1.48 & 1.12\xx{} & [1.06, 1.18] \\
                & 4B   & 10.2 & 3.3 & 2 & 74.7 & 0.66 & 1.66 & \textbf{1.16\xx{}} & [1.09, 1.23] \\
\bottomrule
\end{tabular}
\end{table}

Table~\ref{tab:four} summarizes the four benchmarks for drafts of 0.6B, 1.7B and 4B parameters
trained with one recipe, with Qwen3-32B and with Gemma-4-31B-it as the target. The Gemma rows repeat the
whole pipeline: new sequential traces on the same task splits, the same three drafts trained with the same
recipe on those traces, and the same serving configurations (App.~\ref{app:serving} gives the two serving
details that differ).

With the smallest draft, trained for GPU-minutes on the target's own traces, LEAP speeds up whole tasks by
1.18--1.61\xx{} with Qwen3 and 1.10--1.63\xx{} with Gemma, on planning-only tasks and on tasks that execute
tools. Task success shows no systematic change: across the 28 comparisons of App.~\ref{app:success}, the
speculative arm solves more tasks in 8, fewer in 8 and as many in 12, never differing by more than three.
Every committed action is target-authored, but whole trajectories can differ from the target-only run: they
are identical on 96\% of the OpenAGI and 83\% of the TaskBench tasks, while on the tool benchmarks, whose
steps reason for 10--12\,s, most trajectories take a different decision at some step (Table~\ref{tab:parity}).
On $\tau^2$-bench 37 of the 40 tasks are still scored identically by both arms with Qwen3, and with Gemma the
two arms follow identical trajectories on 52--55 of the 60--63 planning tasks. Repeating the $\tau^2$-bench main run (the
configuration of Table~\ref{tab:model}) end to end gives 1.59\xx{} [1.44, 1.77] against 1.58\xx{}
[1.44, 1.75] in the first pass.

The model of Sec.~\ref{sec:model} explains the ordering across benchmarks and the differences between the
targets. OpenAGI has the highest acceptance and $k{=}3$, so it commits the most steps per round
($c{=}2.05$ with Qwen3), but it also pays the largest end-of-task loss ($e{=}0.38$; tasks are 4 steps
long). $\tau^2$-bench has lower acceptance but long tasks ($e{=}0.11$) and a draft that is cheap next to a
12\,s target step ($k\,t_{\mathrm{draft}}/t_{\mathrm{gen}} = 0.03$). BFCL is limited by structure: its
segments between user turns average 2.6 steps, 26\% of its rounds are width-1 hand-off probes at turn
boundaries, and its second proposal is accepted in only 7\% of rounds, so $c$ is 1.44 and the ceiling
$\min(k{+}1, L)/\beta$ is about 1.7\xx{} even for a perfect draft. Changing the target changes the step
length, and with it the weight of the draft's cost. Gemma's steps are about twice as long as Qwen3's on the
planning benchmarks (2.1 and 2.3\,s against 1.1 and 1.8\,s), which lowers the draft term at a fixed draft
cost, and its drafts agree with it more often (77--85\% against 75--81\%), so its planning speedups are
higher (1.52--1.68\xx{}). Its tool steps are shorter: 2.6 times on $\tau^2$-bench (4.7 against 12.4\,s) and
three times on BFCL (3.3 against 10.8\,s). There the draft term rises from 0.03--0.04 to 0.08--0.12 and
from 0.04--0.07 to 0.11--0.19: on $\tau^2$-bench, with acceptance flat across sizes (0.79--0.80), the
smallest draft is the fastest, and on BFCL the speedups fall to 1.10--1.16\xx{} while the turn structure
that bounds the ceiling is unchanged.

The draft size follows the same logic. With Qwen3 the 1.7B draft is faster on every benchmark
(1.28--1.63\xx{}) and the 4B draft is faster still on BFCL alone. With Gemma the three drafts tie on
OpenAGI (1.55, 1.55 and 1.52\xx{}), where the ratio of committed steps between sizes is 1.00 against a cost
ratio of 1.01; the 1.7B draft is fastest on TaskBench and the 4B draft on BFCL, as with Qwen3.
Sec.~\ref{sec:training} takes this up.

In the $\tau^2$-bench main run of Table~\ref{tab:model}, two tasks end early in the speculative arm (the
re-measurement of Table~\ref{tab:four} has no such case: 40 of 40 terminate). On one context that six
test tasks share, the target, verifying three prefixes in one batch, ran to its
2{,}048-token budget without emitting an action in two of them, while the same
context parsed in the sequential arm and in the other four tasks. We attribute this to the
target's batch-dependent decoding (App.~\ref{app:success}); the commit audit passes on both tasks. The two tasks
are kept in the pooled ratio (their speculative arm is shorter), and excluding
them gives 1.58\xx{} [1.44, 1.75] as well.

\subsection{Framework validation}
\label{app:detail:validation}\label{app:validation}

We test Eq.~\ref{eq:S} as a predictor. For each configuration we take
$c$ from the measured per-position hit rates (Eq.~\ref{eq:c}), replay the
\emph{sequential} arm's per-step generation times to obtain $\beta$ (group
consecutive steps of a segment into rounds with independent hits at rate $q_1$
and take the maximum within each group), and measure $t_{\mathrm{draft}}$. The
prediction uses no timing from the speculative arm's verifications; the hit rates and the draft time are
measured in the speculative run.
Table~\ref{tab:model}, Table~\ref{tab:tool} and Fig.~\ref{fig:pred}(a) compare it with the measured
pooled speedup for twenty-one configurations spanning acceptance $0.02$--$0.79$, depth
$k \in \{1,2,3,5\}$, normalized proposing cost $D \in [0.02, 0.74]$ and added
tool latency up to $0.46$ of a target step.

\begin{figure}[tbp]
\centering
\includegraphics[width=\linewidth]{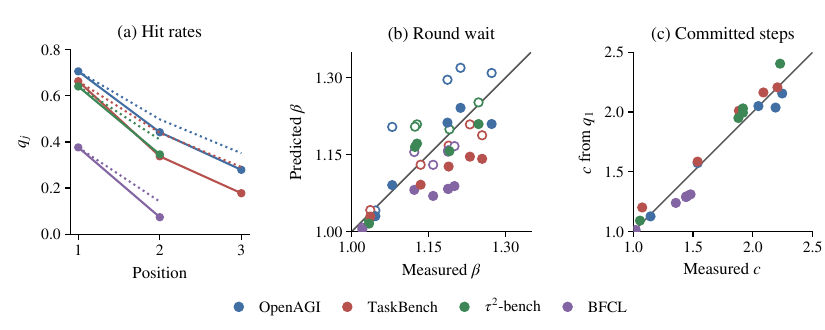}
\caption{\textbf{Component checks of the latency framework.} (a) Measured per-position hit rates (solid)
and $q_1^j$ (dotted) for the four main configurations. (b) Trace-replay $\beta$ (filled) and the
independent-step closed form (hollow), against measured values. (c) Committed steps from the
single-rate replay against measured commits. Configuration details are in Table~\ref{tab:model}.}
\label{fig:components}
\end{figure}

\begin{table}[t]
\centering
\footnotesize
\caption{\textbf{Predicted vs.\ measured speedup.} $q_j$: measured
per-position hit rates; $c$: committed steps per round; $\beta$: round-wait
inflation, predicted from the sequential arm's step times and measured from
the speculative arm; draft: the measured draft time per round divided by the
mean sequential step time (the $D$ term of Eq.~\ref{eq:S}, with rounds that
stop proposing early counted as run). ``i.i.d.''
replaces the measured hit rates by $q_j = q_1^{\,j}$. Every row was run with
cancellation, so the harness never waits on already-discarded verifications
(Sec.~\ref{sec:model}, assumption (e)) and the measured column is the speedup as run.
Rows with other $k$ exercise the depth terms of the model; Table~\ref{tab:tool} exercises its tool
terms. The 32B-direct rows use the target
itself, with reasoning disabled, as the draft; the untrained rows use the 0.6B base without an adapter.
A dash marks a position beyond the row's $k$.}
\label{tab:model}
\setlength{\tabcolsep}{3.2pt}
\resizebox{\linewidth}{!}{%
\begin{tabular}{llrrrrrrrrrrr}
\toprule
 & & & \multicolumn{3}{c}{hit rates} & & \multicolumn{2}{c}{$\beta$} & & \multicolumn{2}{c}{predicted} & \\
\cmidrule(lr){4-6}\cmidrule(lr){8-9}\cmidrule(lr){11-12}
benchmark & draft & $k$ & $q_1$ & $q_2$ & $q_3$ & $c$ & pred. & meas. & draft & i.i.d. & Eq.~\ref{eq:S} & measured \\
\midrule
OpenAGI   & trained 0.6B   & 1 & 0.74 & --   & --   & 1.54 & 1.09 & 1.08 & 0.08 & 1.35 & 1.32 & 1.34 \\
OpenAGI   & trained 0.6B   & 3 & 0.71 & 0.44 & 0.28 & 2.05 & 1.21 & 1.19 & 0.18 & 1.48 & 1.47 & 1.51 \\
OpenAGI   & trained 0.6B   & 5 & 0.70 & 0.47 & 0.31 & 2.25 & 1.24 & 1.21 & 0.22 & 1.47 & 1.54 & 1.57 \\
OpenAGI   & untrained 0.6B & 3 & 0.15 & 0.00 & 0.00 & 1.14 & 1.03 & 1.05 & 0.13 & 0.97 & 0.98 & 0.98 \\
OpenAGI   & 32B-direct     & 3 & 0.70 & 0.50 & 0.33 & 2.19 & 1.21 & 1.27 & 0.61 & 1.12 & 1.20 & 1.16 \\
\midrule
TaskBench & trained 0.6B   & 1 & 0.72 & --   & --   & 1.54 & 1.09 & 1.13 & 0.08 & 1.36 & 1.32 & 1.27 \\
TaskBench & trained 0.6B   & 3 & 0.66 & 0.34 & 0.18 & 1.88 & 1.13 & 1.19 & 0.16 & 1.56 & 1.46 & 1.37 \\
TaskBench & trained 0.6B   & 5 & 0.67 & 0.37 & 0.22 & 2.09 & 1.14 & 1.25 & 0.21 & 1.60 & 1.55 & 1.41 \\
TaskBench & untrained 0.6B & 3 & 0.21 & 0.04 & 0.00 & 1.07 & 1.03 & 1.04 & 0.07 & 1.10 & 0.98 & 0.98 \\
TaskBench & 32B-direct     & 3 & 0.74 & 0.50 & 0.33 & 2.21 & 1.15 & 1.23 & 0.73 & 1.17 & 1.17 & 1.13 \\
\midrule
$\tau^2$-bench & trained 0.6B & 2 & 0.64 & 0.35 & -- & 1.88 & 1.16 & 1.19 & 0.03 & 1.64 & 1.58 & 1.58 \\
$\tau^2$-bench & trained 0.6B & 3 & 0.71 & 0.37 & 0.27 & 2.23 & 1.21 & 1.25 & 0.04 & 1.92 & 1.78 & 1.79 \\
$\tau^2$-bench & untrained 0.6B & 2 & 0.09 & 0.00 & -- & 1.06 & 1.02 & 1.03 & 0.02 & 1.06 & 1.02 & 1.03 \\
\midrule
BFCL      & trained 0.6B   & 1 & 0.36 & --   & --   & 1.36 & 1.07 & 1.16 & 0.02 & 1.13 & 1.24 & 1.15 \\
BFCL      & trained 0.6B   & 2 & 0.38 & 0.07 & --   & 1.45 & 1.08 & 1.19 & 0.04 & 1.16 & 1.30 & 1.20 \\
BFCL      & trained 0.6B   & 3 & 0.39 & 0.07 & 0.02 & 1.48 & 1.09 & 1.20 & 0.04 & 1.16 & 1.31 & 1.19 \\
BFCL      & untrained 0.6B & 2 & 0.02 & 0.00 & --   & 1.02 & 1.00 & 1.02 & 0.05 & 0.96 & 0.97 & 0.95 \\
\bottomrule
\end{tabular}}
\end{table}

Three observations. First, the prediction is within $0.05$ of the measured
speedup for fourteen of the twenty-one configurations and within $0.10$ for all but
TaskBench at $k{=}5$ ($0.14$; harness idle time removed) and BFCL at $k{=}3$ ($0.12$).
This includes the untrained drafts on all four benchmarks (predicted $0.97$--$1.02$, measured
$0.95$--$1.03$). Second, $\beta$ replayed from sequential step times lands within
$0.07$ of the value measured in the speculative arm on fifteen configurations.
It is $0.08$ below it for the 32B-direct draft on TaskBench, $0.09$--$0.11$ below it on BFCL at
every depth without added latency and at $+2$\,s, whose step times are heavy-tailed (coefficient of
variation 1.07 against 0.3--0.5 elsewhere), and $0.11$ below it on TaskBench at $k{=}5$, where
a round issues up to six verification requests but the target server runs at most four at once
(Table~\ref{tab:serving}), so the last ones wait, which the replay does not model;
the closed form of Sec.~\ref{sec:model} is within $0.13$ on all twenty-one
(Fig.~\ref{fig:components}(b)). The largest speedup errors, TaskBench at $k{=}5$ and
BFCL, are where $\beta$ is under-predicted, which supports assumption (b) only
up to the few percent of concurrency cost the replay does not model. Third, the
i.i.d.\ shortcut $q_j = q_1^{\,j}$ is off by up to $0.18$ in $c$ in either
direction (Fig.~\ref{fig:components}(c)): it over-predicts on TaskBench and $\tau^2$-bench,
where rounds often run into the end of the task so that later positions do not
exist, and under-predicts on BFCL,
where a quarter of the rounds are hand-off probes with no proposal; the paper
therefore reports $q_j$ rather than a single acceptance rate.

\paragraph{Tool latency.}

Adding 2 or 5\,s to every executed tool step, in both arms, exercises the tool terms of Eq.~\ref{eq:S}
(Table~\ref{tab:tool}; 0.16--0.46 of a target step). The speedup falls as tools slow down: from 1.58\xx{} to
1.41\xx{} and 1.20\xx{} on $\tau^2$-bench (40 tasks), and from 1.20\xx{} to 1.07\xx{} and 1.02\xx{} on BFCL
(60 tasks). The hit rates and $c$ barely change, so the loss is in the round cost. The proposals' tools run
while proposing, which raises the draft term from 0.03 to 0.31 and 0.74 on $\tau^2$-bench and from 0.04 to
0.23 and 0.50 on BFCL, and the round pays the delay again when the target executes its own action, while the
sequential arm pays it once per step. For the four points with added latency, Eq.~\ref{eq:S} is within 0.08
of the measured speedup, and within 0.02 with the measured $\beta$; on BFCL it counts only the steps that
execute a tool and only the target's own tool steps (App.~\ref{app:predictor}). Task success, sequential against
speculative, is 22 against 22 and 22 against 27 on $\tau^2$-bench, and 34 against 34 and 34 against 33 on BFCL.

\paragraph{Prediction from offline hit rates.}
Table~\ref{tab:model} takes the hit rates from the run it decomposes. To predict a configuration before
running it, we replay the loop along the held-out sequential trajectories with the draft's teacher-forced
predictions of Table~\ref{tab:offline}: a round proposes up to $k$ actions and stops at a predicted
terminal action, the accepted prefix is the run of consecutive matches, and the round commits $A{+}1$
steps ($A$ when the prefix ends the task). Table~\ref{tab:exante} compares the replayed $q_j$ and $c$
with the live values and feeds them, with the trace-replay $\beta$ and the draft's single-stream
latency measured offline on the same held-out contexts (draft served alone, one request at a time;
0.07 of a target step on the planning benchmarks and 0.014 on $\tau^2$-bench), into Eq.~\ref{eq:S}. On OpenAGI
and $\tau^2$-bench the replay reproduces the live hit rates within 0.05 and predicts the speedup within 0.03;
on $\tau^2$-bench the hit rates come from the 14 validation tasks, sequential timing comes from the 40 test tasks, and draft latency is probed offline. On
TaskBench the replay overestimates $q_1$ (0.73 against 0.66) and the speedup (1.56 against 1.37): the
target's live trajectories depart from its sequential trace on
more tasks there (trajectory parity 0.86 against 0.98 on OpenAGI), and the draft is less accurate on the
contexts the target visits only in the live runs. An ex-ante prediction is therefore accurate to the extent that
the target is trajectory-stable under batched decoding; Proposition~\ref{prop:commit} fixes the
teacher-forced distribution only for the trajectory the target actually takes. Run on the Gemma traces
with offline draft-latency probes for 0.6B and measured per-proposal draft latency from the speculative runs for 1.7B and 4B (Table~\ref{tab:prediction_inputs}), the prediction forecasts 1.64--1.77\xx{} against 1.52--1.68\xx{} measured
on the six planning configurations, high by 0.02--0.21, and 1.60--1.62\xx{} against 1.55--1.63\xx{} on
$\tau^2$-bench, where the replay's lower hit rate on the validation tasks (0.64--0.72 against 0.76--0.77
live) offsets the same $\beta$ gap. That gap is the residual on the planning benchmarks: the trace replay
gives 1.16--1.23 where the measured round-wait inflation is 1.23--1.31. It underestimates
verification wait by 0.05--0.13 across these configurations. This gap does not identify a uniform
per-request slowdown, because the awaited contexts and generation lengths can also change.
For Qwen3 the difference is $-0.03$ to $+0.11$ across Table~\ref{tab:model}; the underlying cause
of the Gemma gap is not established.

\begin{table}[t]
\centering
\caption{\textbf{Ex-ante inputs from teacher-forced predictions.} Hit rates and committed steps per
round from replaying the loop on the held-out sequential trajectories with the offline predictions of
Table~\ref{tab:offline} (``replay''), against the live values of the runs of Table~\ref{tab:model}; the
ex-ante speedup uses the replayed values with the trace-replay $\beta$ and the draft's single-stream
latency measured offline on the held-out contexts. $\tau^2$-bench obtains hit rates from the 14 validation tasks and sequential timing from the 40 test tasks. BFCL is omitted: its
rounds are bounded by user turns with a terminal probe, which the replay does not reproduce. A dash marks
a position beyond $k{=}2$.}
\label{tab:exante}
\small
\begin{tabular}{lrrrrrrrrrr}
\toprule
 & \multicolumn{4}{c}{replay, before the run} & \multicolumn{4}{c}{live, in the run} & \multicolumn{2}{c}{speedup} \\
\cmidrule(lr){2-5}\cmidrule(lr){6-9}\cmidrule(lr){10-11}
benchmark & $q_1$ & $q_2$ & $q_3$ & $c$ & $q_1$ & $q_2$ & $q_3$ & $c$ & ex-ante & measured \\
\midrule
OpenAGI        & 0.70 & 0.45 & 0.26 & 2.03 & 0.71 & 0.44 & 0.28 & 2.05 & 1.48\xx{} & 1.51\xx{} \\
TaskBench      & 0.73 & 0.42 & 0.24 & 2.04 & 0.66 & 0.34 & 0.18 & 1.88 & 1.56\xx{} & 1.37\xx{} \\
$\tau^2$-bench & 0.59 & 0.34 & --   & 1.85 & 0.64 & 0.35 & --   & 1.88 & 1.59\xx{} & 1.58\xx{} \\
\bottomrule
\end{tabular}
\end{table}

\paragraph{Segment truncation.}
Appendix~\ref{app:length} buckets the tasks of the main runs by length. The truncation loss $e$ of
Eq.~\ref{eq:c} falls with task length as the equation says (from 0.52 to 0.30 on OpenAGI, from 0.20 to 0.07
on $\tau^2$-bench). The per-bucket speedup does not fall with it, because the draft is most often right on the
short tasks; a prediction that ignores this correlation misses by up to 0.4 in the extreme buckets, while
the run-level prediction is within 0.10.

\paragraph{Self-speculation loses on cost, not on acceptance.}
\label{app:selfspec}
The 32B-direct rows use the target with reasoning disabled as its own draft,
issued as separate requests to the same server: the self-drafting idea
of~\citet{spork,swyr} without the prefix forking and cache sharing of those systems.
Its acceptance is at least as high as the trained 0.6B draft's (0.78 vs.\
0.75 on OpenAGI, 0.79 vs.\ 0.68 on TaskBench) and it commits more steps per
round, yet its speedup is 1.16\xx{} and 1.13\xx{}: each proposal costs
$3.3$--$4.0\times$ as much wall-clock time, which raises the draft term of
Eq.~\ref{eq:S} from $0.16$--$0.18$ to $0.61$--$0.73$, and sharing the target's
server raises $\beta$ (1.19$\to$1.27 on OpenAGI). The trained 0.6B draft
reaches comparable acceptance at about a quarter of the drafting cost. On the
tool benchmarks the reasoning-free target is not even the better predictor: its
offline agreement with the target's reasoning-on decisions is 42.5\% on $\tau^2$-bench
and 62.6\% on BFCL, against 63.7\% and 66.7\% for the trained draft
(Table~\ref{tab:offline}), so self-speculation would lose there on both terms
of Eq.~\ref{eq:S}. The margin narrows with the target's step length, as the draft term predicts: on the
Gemma target of Table~\ref{tab:four}, whose planning steps are twice as long, a self-proposal costs
0.12--0.18 of a step, and self-speculation reaches 1.48\xx{} [1.23, 1.75] on OpenAGI and 1.54\xx{}
[1.46, 1.64] on TaskBench at acceptance 0.90--0.91 (draft term 0.23--0.47), against 1.55\xx{} and
1.60--1.68\xx{} for the trained drafts at a draft term of 0.09--0.23.

\paragraph{What it costs.}
Speculation trades target compute for latency. On $\tau^2$-bench the speculative arm issued 789 target requests
against 559 sequential steps and generated 1.16\xx{} the completion tokens as returned to the client
(a cancelled verification counts the tokens generated up to the cancel; the server aborts it on
disconnect); the draft adds about a tenth of the GPU memory and 3\% of the round time. The trade is
favourable when per-request latency is the constraint (interactive use, evaluation and RL rollouts) and
unfavourable for throughput-bound serving.

\begin{table}[tbp]
\centering\small
\caption{\textbf{Predictions using offline hit rates, both targets.} The eighteen size-study points
in Fig.~\ref{fig:pred}(b), using each draft's offline predictions and the original trace-replay $\beta$.
Serving configurations and measured speedups match Table~\ref{tab:main_results}. Draft timing is
either an offline single-stream probe or measured per-proposal latency from the speculative run
(``run''); the latter rows are conditional on that cost. Sequential timing is from the test trajectories;
$\tau^2$-bench hit rates are from validation trajectories. The last two columns give the trace-replay
and measured $\beta$ plotted in Fig.~\ref{fig:pred}(c).}
\label{tab:prediction_inputs}
\begin{tabular}{lllcrrrr}
\toprule
Target & Benchmark & Draft & Draft timing & Predicted & Measured & $\beta$ replay & $\beta$ meas. \\
\midrule
Qwen3 & OpenAGI & 0.6B & probe & 1.48 & 1.50 & 1.22 & 1.19 \\
Qwen3 & OpenAGI & 1.7B & run & 1.53 & 1.55 & 1.22 & 1.21 \\
Qwen3 & OpenAGI & 4B & run & 1.47 & 1.46 & 1.24 & 1.25 \\
Qwen3 & TaskBench & 0.6B & probe & 1.53 & 1.38 & 1.14 & 1.22 \\
Qwen3 & TaskBench & 1.7B & run & 1.58 & 1.47 & 1.16 & 1.21 \\
Qwen3 & TaskBench & 4B & run & 1.53 & 1.36 & 1.16 & 1.26 \\
Qwen3 & $\tau^2$-bench & 0.6B & probe & 1.58 & 1.61 & 1.14 & 1.15 \\
Qwen3 & $\tau^2$-bench & 1.7B & probe & 1.64 & 1.63 & 1.15 & 1.17 \\
Qwen3 & $\tau^2$-bench & 4B & probe & 1.62 & 1.56 & 1.16 & 1.16 \\
Gemma & OpenAGI & 0.6B & probe & 1.77 & 1.55 & 1.18 & 1.31 \\
Gemma & OpenAGI & 1.7B & run & 1.66 & 1.55 & 1.16 & 1.23 \\
Gemma & OpenAGI & 4B & run & 1.68 & 1.52 & 1.17 & 1.26 \\
Gemma & TaskBench & 0.6B & probe & 1.68 & 1.60 & 1.20 & 1.29 \\
Gemma & TaskBench & 1.7B & run & 1.74 & 1.68 & 1.23 & 1.28 \\
Gemma & TaskBench & 4B & run & 1.64 & 1.62 & 1.23 & 1.28 \\
Gemma & $\tau^2$-bench & 0.6B & probe & 1.62 & 1.63 & 1.21 & 1.34 \\
Gemma & $\tau^2$-bench & 1.7B & run & 1.61 & 1.55 & 1.25 & 1.34 \\
Gemma & $\tau^2$-bench & 4B & run & 1.60 & 1.55 & 1.24 & 1.35 \\
\bottomrule
\end{tabular}
\end{table}

%% file: sec_training.tex
\subsection{Training the draft: recipe, data, input and size}
\label{app:detail:training}
The main results use one recipe per benchmark (Sec.~\ref{sec:train}). This section varies the recipe
on OpenAGI, where the held-out test set (63 tasks, 263 contexts, scored three times each at temperature
0) is cheap to re-score offline, and then varies the draft size across benchmarks. We split the OpenAGI
test tasks into \emph{seen} (40 tasks whose committed action sequence also occurs in a training task) and
\emph{unseen} (the other 23), because the two respond differently to every knob below.

\begin{table}[t]
\centering
\caption{OpenAGI offline agreement (\%) of 0.6B drafts under recipe variations. Rows marked ``ep\,7''
train for 12 epochs and pick the checkpoint by validation exact match; the Table~\ref{tab:offline} draft
trains for 3 epochs and picks by validation loss (epoch 2). Seen/unseen as defined in the text.}
\label{tab:recipe}
\small
\begin{tabular}{lrrr}
\toprule
draft & overall & seen & unseen \\
\midrule
untrained 0.6B & 12.8 & 11.5 & 14.8 \\
LoRA $r{=}64$, 3 epochs, epoch 2 by validation loss (Table~\ref{tab:offline}) & 75.0 & 82.2 & 64.5 \\
LoRA $r{=}64$, ep\,7 by validation exact match & \textbf{77.1} & \textbf{86.4} & 63.2 \\
LoRA $r{=}64$, epoch 12 & 74.1 & 83.4 & 60.4 \\
full fine-tuning, ep\,7 & 75.2 & 84.7 & 61.0 \\
LoRA $r{=}64$, ep\,7, 50\% of the pairs (535) & 72.0 & 80.5 & 59.4 \\
LoRA $r{=}64$, ep\,7, 25\% of the pairs (270) & 71.5 & 74.7 & \textbf{66.7} \\
\midrule
32B-direct (target, reasoning off) & 76.6 & 76.6 & 76.4 \\
target self-consistency & 97.3 & 98.1 & 96.2 \\
\bottomrule
\end{tabular}
\end{table}

\paragraph{Epochs, capacity and data.}
Table~\ref{tab:recipe} gives the offline agreement. Three findings. (i)~Training longer raises agreement
on seen sequences only: from epoch 2 to epoch 7 the seen score rises from 82.2 to 86.4 while the unseen
score falls from 64.5 to 63.2, and epoch 12 loses on both. (ii)~Full fine-tuning of the 0.6B model does
not beat the rank-64 adapter (75.2 against 77.1). (iii)~Halving or quartering the training pairs lowers
the overall score by five to six points, almost entirely on seen tasks; the unseen score is not monotone in the
amount of data. The two untrained references are flat across the split (32B-direct 76.6/76.4, target
self-consistency 98.1/96.2), so the seen/unseen gap is a property of the trained draft: most of what more
epochs and more data buy is familiarity with training sequences, and the unseen level of 63--67 is not
moved by these knobs. End to end the difference is small. At $k{=}3$ the epoch-7 draft reaches
1.476\xx{} against 1.455\xx{} for the epoch-2 draft, and an independent repeat of the epoch-2 run gave
1.477\xx{}; the live acceptance splits the same way (epoch 2: 82.5\% on seen and 67.0\% on unseen
tasks, 1.53\xx{} and 1.35\xx{}; epoch 7: 86.5\% and 65.0\%, 1.63\xx{} and 1.28\xx{}).

\paragraph{Draft input.}
The draft sees the task and the committed actions. Adding the target's reasoning text for the committed
steps to the draft's input raises the unseen score by at most 2.5 points (64.5 to 67.0, with the epoch-4
checkpoint) at equal overall agreement. In the live runs, however, the draft proposes positions $2..k$ before the
target has produced the reasoning for its own proposals, so training must drop the trailing one or two
thoughts; with that dropout the unseen score falls to 58--60. We therefore keep the action-only input.

\paragraph{Transfer across benchmarks.}
A trained draft is specific to its benchmark (Table~\ref{tab:transfer} in App.~\ref{app:transfer}). On each
of the other three benchmarks, the four trained drafts agree with the target about as often as the untrained
base does, within about four points. The only exception is between the two planning benchmarks: the TaskBench
draft reaches 23.3\% on OpenAGI and the OpenAGI draft 29.6\% on TaskBench, against 13.3\% and 25.3\% for the
base. This is still far below the 75\% of a draft trained on the benchmark itself. Each workload therefore
needs its own draft; training one takes minutes.

\paragraph{Draft size.}
The three drafts of Table~\ref{tab:four} bound the size axis: 0.6B is the smallest base we train and 4B the
largest that shares one H100 with the 32B target at the benchmarks' context lengths. Table~\ref{tab:size} and
Fig.~\ref{fig:size} show what varies. Offline agreement mostly rises with size
(0.6B to 4B: 75.4 to 78.7 on OpenAGI, 74.7 to 80.5 on TaskBench, where 1.7B already reaches 81.0, 63.3 to 69.5 on
$\tau^2$-bench, 66.7 to 79.8 on BFCL), and the draft's cost always does: its per-round term $D$ (Eq.~\ref{eq:S}) grows from 0.16--0.17 to
0.28--0.29 on the planning benchmarks and from 0.03--0.04 to 0.04--0.07 on the tool benchmarks, and sharing
the GPU with the 4B draft raises $\beta$ by up to 0.06. The speedup is the balance of the two. On OpenAGI and
TaskBench, where a target step takes 1--2\,s, 1.7B is the fastest draft (1.55 and 1.47\xx{} against 1.50 and
1.38\xx{} for 0.6B) and 4B falls back to the 0.6B level (1.46 and 1.36\xx{}) although it agrees most; two
independent runs of every cell give the same ordering, and Eq.~\ref{eq:S} predicts it from the sequential
trace (1.48, 1.53, 1.47 on OpenAGI; 1.53, 1.58, 1.53 on TaskBench), with TaskBench
overpredicted by 0.11--0.17. These predictions replay each draft's offline hits using the inputs in
Table~\ref{tab:prediction_inputs}; TaskBench's replayed $\beta$ is 0.05--0.10 below the measured one.
On $\tau^2$-bench, where a target step takes 12\,s and $D$ is 0.03, predicted speedups are 1.58,
1.64 and 1.62; the three drafts measure 1.61, 1.63 and 1.56\xx{}, within one another's bootstrap intervals,
with the 4B draft's calls 1.55\xx{} as long as the 0.6B draft's. On BFCL, where the 0.6B draft agrees least
and a target step takes 11\,s, every size step pays: 1.18, 1.28 and 1.34\xx{} for agreement 66.7, 72.7 and
79.8. The bootstrap intervals of neighbouring cells overlap (OpenAGI 1.7B [1.43, 1.68] against 4B [1.30,
1.62]; BFCL 1.7B [1.20, 1.36] against 4B [1.21, 1.48]), so the ordering rests on its consistency across
runs and benchmarks and on the model rather than on any single pair.

A larger draft pays exactly when its extra agreement outweighs its extra cost, and Eq.~\ref{eq:S} says
where that is: with $c$ the committed steps per round and $\beta + D$ the round cost, draft $B$ beats draft
$A$ when $c_B/c_A > (\beta_B + D_B)/(\beta_A + D_A)$. With the measured values, 4B against 0.6B has a cost
ratio of 1.12 on both planning benchmarks against a gain in $c$ of 1.09--1.10, and on BFCL a cost ratio of
1.03 against a gain of 1.16; on $\tau^2$-bench both ratios are within 0.02 of one, and the three speedups are equal
within noise. The best size therefore depends on the target step. With 1--2\,s steps a mid-size draft is
fastest; with 10--12\,s steps the draft's cost weighs little, and the larger drafts are faster on BFCL and
within noise on $\tau^2$-bench. Training cost grows with size (Table~\ref{tab:size}). The 0.6B draft, trained in
about a minute on OpenAGI, is within 0.1\xx{} of the fastest size in seven of the eight benchmark--target
pairs of Table~\ref{tab:four}, which is why it is our default and the draft of the model validation
(Sec.~\ref{sec:validation}).
For practitioners, we recommend starting with a trained 0.6B draft and $k{=}1$.

\begin{table}[t]
\centering
\caption{\textbf{Draft size.} Three drafts trained with the same recipe. Offline agreement (\%) on the
held-out contexts of Table~\ref{tab:offline}; the measured per-round draft term $D$ of Eq.~\ref{eq:S} in the
common serving configurations of Table~\ref{tab:four}; the speedup Eq.~\ref{eq:S} predicts from the
sequential trace and each draft's own offline hit predictions, with the draft timing sources
listed in Table~\ref{tab:prediction_inputs}. BFCL has no offline-hit prediction (Fig.~\ref{fig:size}).
Measured speedups are in Table~\ref{tab:four}.}
\label{tab:size}
\footnotesize\setlength{\tabcolsep}{3pt}
\resizebox{\linewidth}{!}{%
\resizebox{\linewidth}{!}{%
\begin{tabular}{lrrrrrrrrrrr}
\toprule
 & \multicolumn{4}{c}{offline agreement} & \multicolumn{4}{c}{draft term $D$} & \multicolumn{3}{c}{predicted speedup} \\
\cmidrule(lr){2-5}\cmidrule(lr){6-9}\cmidrule(lr){10-12}
draft & OpenAGI & TaskBench & $\tau^2$-bench & BFCL & OpenAGI & TaskBench & $\tau^2$-bench & BFCL & OpenAGI & TaskBench & $\tau^2$-bench \\
\midrule
0.6B & 75.4 & 74.7 & 63.3 & 66.7 & 0.17 & 0.16 & 0.03 & 0.04 & 1.48\xx{} & 1.53\xx{} & 1.58\xx{} \\
1.7B & 77.6 & \textbf{81.0} & 68.1 & 72.7 & 0.20 & 0.19 & 0.03 & 0.04 & \textbf{1.53\xx{}} & \textbf{1.58\xx{}} & \textbf{1.64\xx{}} \\
4B   & \textbf{78.7} & 80.5 & \textbf{69.5} & \textbf{79.8} & 0.28 & 0.29 & 0.04 & 0.07 & 1.47\xx{} & 1.53\xx{} & 1.62\xx{} \\
\bottomrule
\end{tabular}}}
\end{table}

\paragraph{Objective.}
Sec.~\ref{sec:train} argues that cross-entropy on the target's committed actions is a natural
objective. Table~\ref{tab:objective} lists the comparison that tests alternatives on the two benchmarks with
ground truth, TaskBench and BFCL: supervision from gold labels instead of the target's actions (does the
draft need to imitate the target rather than be correct?) and reinforcement learning from the exact-match
reward starting from the SFT draft. Offline agreement is measured for every variant; the end-to-end speedup
only where the offline gain over SFT exceeds three points; no variant reaches it. On BFCL the gold-label draft is trained with the
same recipe on the benchmark's ground-truth calls replayed in the real environment, one call per decision
(130 tasks, 1{,}200 pairs, 39\% of them turn ends; the target's traces give 1{,}254 pairs, 37\% turn ends).
It reaches 50.5\% decision agreement against 66.7\%: on the 63 held-out decisions where the target calls
a tool it ends the turn 26 times and matches the target's tool group 14 times (22\%), against 3 and
31 (49\%) for the draft trained on the target's own actions; on the 51 single-call decisions alone it matches
14 against 29. Correctness by the benchmark's standard and agreement with the target are different
objectives. The comparison changes the source of the training data, not only its labels: the gold contexts
follow the gold calls, and on BFCL the gold data has one call per decision where the target groups calls.
On TaskBench the gold-label draft reaches 69.9\% against 75.2\%. The gap is smaller there because
the annotated tool sequences often coincide with the target's: they are identical on 61 of the 228 training
tasks and use the same tools, with other argument strings, on 108 more. Reinforcement learning from the exact-match reward, run with GRPO from the SFT draft on the same
1{,}011 training contexts (eight samples per context, reward 1 for the target's action and 0.1 for a call of
the right tool with other arguments, KL to the SFT draft, two epochs, 48 GPU-minutes), raises TaskBench
agreement from 75.2\% to 77.1\%: nine of the 277 held-out contexts change from miss to hit and four from hit
to miss, a net gain that is not significant at this size and that Eq.~\ref{eq:S} values at $+0.04\times$
(replayed $c$ 2.04$\to$2.13), below the three points at which we run the system. The sampled reward on the
training contexts rises from 0.78 to 0.85 while the greedy output changes on 13 contexts, consistent with the reward
sharpening the SFT draft's distribution around actions it already ranks first. Started from the untrained base
instead (25.0\% agreement, 64\% of outputs parsable), the same GRPO reaches 75.9\% after three epochs,
indistinguishable from the SFT draft (12 contexts gained, 10 lost) and 1.2 points below GRPO from it, at
72 GPU-minutes against 4 for SFT: in this setting the reward adds nothing over cross-entropy on the
target's actions, which reaches the same agreement at a fraction of the cost. On BFCL, GRPO from the SFT draft
(reward 1 for the target's exact tool group or turn end, 1{,}254 contexts, two epochs) gains nothing either:
65.7\% decision agreement against 66.7\%, with the exact tool group up from 49.2\% to 50.8\% and turn ends
recalled less often (91.7\% against 97.2\%). This is consistent
with Sec.~\ref{sec:train}: only a proposal made on a prefix the target itself has committed can ever
be committed, so the contexts that decide acceptance are the target's own, and the teacher-forced pairs
already come from them. The tests show no significant gain from GRPO at 12 to 18 times the training cost on
TaskBench; they do not rule out gains from other objectives or larger budgets. An earlier pilot on
other benchmarks found the same one to two points on structured tool use and a loss on open-vocabulary actions.

\subsection{Online training of the draft while the agent runs}
\label{app:detail:online}
The drafts above were trained after a separate collection pass: we ran the target alone over the training
tasks and logged its actions. That pass is not needed. By Proposition~\ref{prop:commit}, every action the
speculative system commits was generated by the target, whatever the draft proposed. So the commits of the
running system are the same supervision, and the draft can be trained from them between batches of tasks.

We test this from a cold start. The untrained 0.6B draft starts the speculative system on the benchmark's
training tasks, in batches (OpenAGI: 250 tasks in 10 batches of 25; $\tau^2$-bench: 60 tasks in 6 batches of 10).
After each batch we train a fresh adapter on all commits so far with the recipe of Table~\ref{tab:recipe},
on a second GPU, and load it into the running draft server without a restart. Batch $b$ therefore runs with
the adapter trained on the batches before it. Each adapter is scored for offline agreement on the held-out
contexts of Table~\ref{tab:offline}, and the final adapter is measured end to end on the test tasks as in
Table~\ref{tab:four}.

\begin{figure}[t]
\centering
\includegraphics[width=\linewidth]{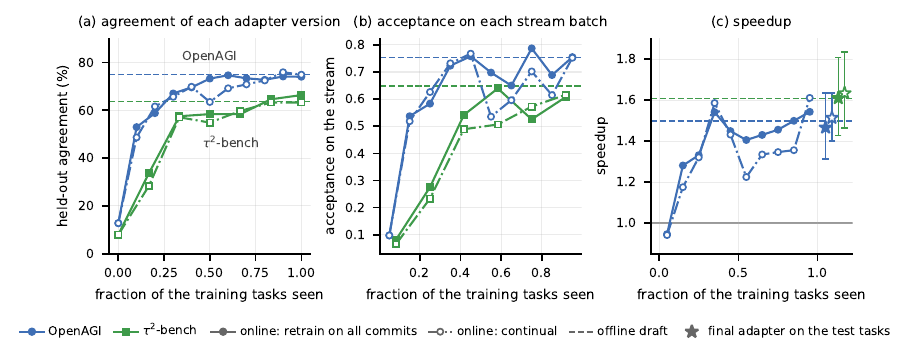}
\caption{\textbf{Online training of the draft.} The draft is trained online, from the commits of the running
speculative system: the untrained 0.6B draft starts the system on the training tasks, and after every
batch a new adapter replaces the previous one in the draft server without a restart. Colour marks the
benchmark. Solid lines retrain the adapter on all commits so far; dash-dot lines with hollow markers
continue the previous adapter (continual). (a) Offline agreement of each adapter version on the held-out
contexts. (b) Acceptance on each batch of the stream. (c) Speedup on each OpenAGI batch and, as stars with
bootstrap intervals, the final adapters on the test tasks. Dashed: the offline drafts
(Tables~\ref{tab:offline} and~\ref{tab:four}); grey line: no speedup.}
\label{fig:online_stream}
\end{figure}

Figures~\ref{fig:online} and~\ref{fig:online_stream} show the result. On OpenAGI the draft agrees on 53\% of the held-out contexts
after 25 tasks, 67\% after 75, and 74.0\% after all 250 (offline draft: 75.0\%). The speedup on the stream
follows: 0.95\xx{} on the first batch with the untrained draft, 1.28\xx{} on the second, and 1.4--1.5\xx{}
from the fourth batch on. The final adapter gives 1.47\xx{} [1.31, 1.63] on the test tasks, against
1.51\xx{} [1.38, 1.64] for the offline draft under the same serving configuration. On $\tau^2$-bench, with tools executed, agreement on the validation
contexts goes from 8.0\% to 66.4\% (offline draft: 63.7\%), acceptance on the stream reaches the offline
draft's level after 30 tasks, and the final adapter gives 1.61\xx{} [1.43, 1.81] on the test tasks, the same
as the offline draft's 1.61\xx{} [1.44, 1.79]. The buffer ends with 1,079 pairs on OpenAGI, against 1,090 in
the offline training set, because the committed trajectories are the target's own.

Retraining on the whole buffer costs more with every batch (27 to 73 seconds per version on OpenAGI, 3 to
18 minutes on $\tau^2$-bench). A continual variant instead continues the previous adapter on the new batch plus an
equal-size sample of earlier pairs, for two epochs, at a fixed 22 seconds per update on OpenAGI. It reaches
the same endpoint (74.9\% agreement, 1.51\xx{} [1.40, 1.64] on the test tasks) but dips once on the way
(63.5\% and 1.22\xx{} at the fifth batch) before recovering. On $\tau^2$-bench it takes 3 to 6 minutes per update
instead of 3 to 18, ends at 63.3\% agreement (offline draft: 63.7\%), and gives 1.63\xx{} [1.46, 1.83] on the
test tasks. Appendix~\ref{app:online} gives the details.

%% file: sec_appendix.tex
\section{Benchmark implementation details}
\label{app:bench}

All four benchmarks share the same loop (Sec.~\ref{sec:loop}) and the same
acceptance test: a proposal is accepted if its canonical string equals the
canonical string of the target's action on the same context. Canonical strings
are produced by the benchmark's parser from the raw response; a response that
does not parse (no action, several actions where one is required, an unknown
tool, or arguments that violate the tool schema) never matches. The target
decodes greedily (temperature 0, seed 0) in every end-to-end run. Table
\ref{tab:bench} lists the per-benchmark settings; the paragraphs below give the
formats.

\begin{table}[h]
\centering
\footnotesize
\setlength{\tabcolsep}{4pt}
\caption{Per-benchmark settings of the end-to-end runs.}
\label{tab:bench}
\resizebox{\linewidth}{!}{%
\begin{tabular}{>{\raggedright\arraybackslash}p{2.5cm}>{\raggedright\arraybackslash}p{2.9cm}>{\raggedright\arraybackslash}p{3.1cm}>{\raggedright\arraybackslash}p{3.0cm}>{\raggedright\arraybackslash}p{3.3cm}}
\toprule
 & OpenAGI & TaskBench & $\tau^2$-bench & BFCL multi-turn \\
\midrule
task type & tool-selection plan & tool calls with args & ticket-driven support & multi-turn function calling \\
tools & 14 $+$ TERMINATE & DailyLife APIs & 44 & per-task catalog \\
target format & ReAct, \texttt{<tool>} tag with tool name & ReAct, \texttt{<tool>} tag with name(\{json\}) & single-call JSON, guided & \texttt{<tool\_call>} tags, ordered group \\
target reasoning & thought $\le 512$ tok. & thought $\le 512$ tok. & on, budget 2{,}048 & on, budget 4{,}096 \\
target action budget & 32 tok. & 128 tok. & (in budget) & (in budget) \\
draft format & direct, tool name & direct, name$+$args & grammar, JSON call & grammar, call group / \texttt{<respond/>} \\
draft budget & 32 tok. & 128 tok. & 512 tok. & 512 tok. \\
observation & none (plan-first) & none (plan-first) & tool result (local backend) & tool result (cloned runtime) \\
segment boundary & task end & task end & task end (\texttt{done}) & user turn \\
$k$ & 3 & 3 & 2 & 2 \\
max steps & 12 & 12 & 60 & 20 per turn \\
train / test tasks & 250 / 63 & 228 / 60 & 60 (+14 val) / 40 & 130 / 60 \\
training pairs & 1{,}090 & 1{,}120 & 932 & 1{,}254 \\
\bottomrule
\end{tabular}}
\end{table}

\paragraph{OpenAGI.}
The agent selects, one step at a time, which of 14 domain-expert tools to
apply to a task description~\citep{openagi}, and emits \texttt{TERMINATE} when
the plan is complete. The target runs ReAct: a thought of at most 512 tokens
followed by an action of at most 32 tokens in the form
\texttt{<tool>Image Denoising</tool>}; the parser takes the last tag pair so
that a preamble naming other tools does not win over the committed action. The
canonical action is the tool name. Steps do not observe tool outputs, so the
context of step $i$ is the task and the committed actions $a_{<i}$ (the target's
thoughts are not part of the context and are not shown to the draft). A task
ends at \texttt{TERMINATE} or after 12 steps. The 313 tasks are split by task
into 250 training and 63 test tasks (seed 42); running the target once over the
training tasks yields 1{,}090 (context, action) pairs, and the test tasks give
263 held-out contexts for the offline agreement in Table~\ref{tab:offline}. The
draft receives the same context rendered as a direct prompt and answers with
the action tag only.

\paragraph{TaskBench.}
We use the DailyLife API subset of TaskBench~\citep{taskbench}, where each
step is a tool call with JSON
arguments, or \texttt{TERMINATE}; for example\par\noindent\hspace*{1.5em}%
\texttt{<tool>apply\_for\_passport(\{"country":"Germany"\})</tool>}\par\noindent The 300 tasks were split 240 / 60 by task before any
model call. The target runs ReAct as on OpenAGI (thought $\le 512$ tokens, action
$\le 128$ tokens). The parser is strict: the response must contain exactly one
action, the tool must exist, the argument names must equal the schema's
parameter set, and every argument value must be a string; anything else is a
parse or schema failure and cannot match. The canonical action is
\texttt{name(\{sorted-key JSON\})}. The training set is derived
deterministically from the target's raw responses on the training tasks
(228 tasks with a valid trajectory, 1{,}120 pairs; 1{,}011 / 109 split by task
for validation); the same derivation on the test tasks gives 57 tasks / 277
contexts for offline agreement, and all 60 test tasks are run end to end. No
test task's tool sequence occurs in training.

\paragraph{$\tau^2$-bench (telecom domain).}
We use the telecom domain of $\tau^2$-bench~\citep{tau2bench} in its
\emph{solo} configuration: there is no simulated user; the agent works from the
task ticket, a policy document, and 44 tools whose implementations run against
the benchmark's local state. We compact the policy text (same rules, shorter
wording) so that contexts fit the serving window; the compact policy is
identical for both arms and both models. The target receives the tool schemas
natively and must emit exactly one call per step as a JSON object
\texttt{\{"name":\,\ldots,"arguments":\,\ldots\}}, constrained by a
\texttt{oneOf} schema over the 44 tools (guided decoding), with reasoning
enabled and a 2{,}048-token budget; a task ends with the \texttt{done} call or
after 60 steps, and success is the benchmark's own reward function
(\texttt{evaluate\_simulation}, solo mode). The canonical action is the JSON
object with sorted keys. Training and validation pairs were collected from the
target at its default sampling for this mode (temperature 0.6, top-$p$ 0.95,
top-$k$ 20, seed 0); end-to-end runs use temperature 0. The draft sees the same
context rendered as a direct prompt and decodes a \texttt{<tool>\{json\}</tool>}
block under a grammar derived from the same schemas (512-token budget, no
reasoning). Each proposal is executed in a cloned environment branch so that
the next proposal sees its observation; a rejected branch is discarded.

\paragraph{BFCL multi-turn.}
We use the multi-turn split of the Berkeley Function Calling
Leaderboard~\citep{bfcl}, 130 training and 60 test tasks.
Each task is a conversation; in each user turn the agent issues one or more
ordered groups of function calls and ends the turn with a natural-language
reply. The target emits calls as \texttt{<tool\_call>\{json\}</tool\_call>}
tags (several tags form one ordered group, which is one decision), with
reasoning enabled, a 4{,}096-token budget, and at most 20 decisions per turn.
The draft predicts the next decision as
\texttt{<tool>[\{call\},\ldots]</tool>} or the literal \texttt{<respond/>},
decoding under a GBNF grammar built from the tools visible in the context
(512-token budget, no reasoning); it never predicts reply text. A decision
matches when the ordered call group is identical after canonicalization (JSON
with sorted keys) or when both sides end the turn. Proposals execute in a
cloned tool runtime (deep-copied instances); the committed branch's state is
the one the next turn sees. User turns are segment boundaries: at the end of
each turn a width-1 probe lets the target confirm the turn end. Official
success is computed by the official BFCL evaluator on the final state and
responses. Verification uses prefix-ordered cancellation
(App.~\ref{app:serving}); the 60 test tasks were run twice end to end, one task
at a time. On the 10-task development split, run twice at each depth before
the test runs, the speculative arm was faster at $k{=}1$ and $k{=}2$
(1.14/1.23\xx{} and 1.18/1.24\xx{}) but completed fewer tasks (9 against 14 at
$k{=}1$, 13 against 14 at $k{=}2$); we selected $k{=}2$ and report the test
split, where success was higher in both passes (Table~\ref{tab:four}).

\section{Trace-replay predictor and idle-time accounting}
\label{app:predictor}

\paragraph{Inputs.}
The predictor of Sec.~\ref{sec:validation} uses (i) the sequential arm's
per-step generation times, grouped by segment; (ii) the per-position hit rates
$q_j$; and (iii) the mean draft time per round from the speculative arm's draft
records. It uses no timing from the speculative arm's verifications.

\paragraph{Replay.}
For every segment of every task, the replay walks the sequential steps in
order. At the start of a round with $r$ steps remaining it draws the number of
leading hits $A$ from independent Bernoulli trials with rate $q_1$, capped at
$\min(k, r-1)$; the round commits $A{+}1$ steps and its wait is the maximum of
those steps' sequential generation times. The predicted $\beta$ is the sum of
round waits divided by the number of rounds times $t_{\mathrm{gen}}$, averaged
over 50 replays. The commits per round $c$ in Table~\ref{tab:model} is the
measured value (Eq.~\ref{eq:c} with the measured $q_j$ and $e$); the
i.i.d.\ column uses the replay's own $A$ instead. Table~\ref{tab:qj} lists the
measured $q_j$ and $e$ for the nine configurations.

\begin{table}[h]
\centering
\small
\caption{Measured per-position hit rates $q_j = \Pr[A \ge j]$, end-of-segment
loss $e$, and committed steps per round $c = 1 + \sum_j q_j - e$. A dash marks a position beyond the row's $k$.}
\label{tab:qj}
\begin{tabular}{llrrrrrrrr}
\toprule
benchmark & draft & $k$ & $q_1$ & $q_2$ & $q_3$ & $q_4$ & $q_5$ & $e$ & $c$ \\
\midrule
OpenAGI   & trained 0.6B   & 3 & 0.71 & 0.44 & 0.28 & --   & --   & 0.38 & 2.05 \\
OpenAGI   & trained 0.6B   & 1 & 0.74 & --   & --   & --   & --   & 0.20 & 1.54 \\
OpenAGI   & trained 0.6B   & 5 & 0.70 & 0.47 & 0.31 & 0.14 & 0.08 & 0.46 & 2.25 \\
OpenAGI   & untrained 0.6B & 3 & 0.15 & 0.00 & 0.00 & --   & --   & 0.01 & 1.14 \\
OpenAGI   & 32B-direct     & 3 & 0.70 & 0.50 & 0.33 & --   & --   & 0.34 & 2.19 \\
TaskBench & trained 0.6B   & 3 & 0.66 & 0.34 & 0.18 & --   & --   & 0.30 & 1.88 \\
TaskBench & 32B-direct     & 3 & 0.74 & 0.50 & 0.33 & --   & --   & 0.36 & 2.21 \\
$\tau^2$-bench   & trained 0.6B   & 2 & 0.64 & 0.35 & --   & --   & --   & 0.11 & 1.88 \\
BFCL      & trained 0.6B   & 2 & 0.38 & 0.07 & --   & --   & --   & 0.00 & 1.45 \\
\midrule
\multicolumn{10}{l}{\emph{Draft-size study, common serving configuration (Sec.~\ref{sec:training})}} \\
OpenAGI   & trained 0.6B   & 3 & 0.69 & 0.44 & 0.27 & --   & --   & 0.37 & 2.04 \\
OpenAGI   & trained 1.7B   & 3 & 0.74 & 0.50 & 0.33 & --   & --   & 0.37 & 2.20 \\
OpenAGI   & trained 4B     & 3 & 0.79 & 0.50 & 0.31 & --   & --   & 0.34 & 2.26 \\
TaskBench & trained 0.6B   & 3 & 0.67 & 0.37 & 0.20 & --   & --   & 0.30 & 1.94 \\
TaskBench & trained 1.7B   & 3 & 0.69 & 0.42 & 0.28 & --   & --   & 0.34 & 2.05 \\
TaskBench & trained 4B     & 3 & 0.71 & 0.43 & 0.27 & --   & --   & 0.31 & 2.10 \\
\bottomrule
\end{tabular}
\end{table}

\paragraph{Tool terms.}
With a delay $\delta$ added to every executed tool step and $\tau = \delta / t_{\mathrm{gen}}$, the sequential
arm pays $\tau$ (in target steps) on every step that executes a tool. A round pays it for the proposals' tools,
which run while proposing and are part of the measured draft term $D$, and for every tool step the target
itself executes. Eq.~\ref{eq:S} becomes $S = c\,(1 + f\tau) / (\beta + D + g\tau)$, with $f$ the share of
sequential steps that execute a tool and $g$ the number of tool steps per round that the target executes. On
$\tau^2$-bench every step but the final one executes a tool and each round ends with one target action, and we
use $f = g = 1$. On BFCL, turn ends execute no tool ($f = 0.62$), accepted proposals are adopted rather than
executed again, so only the target's corrections and bonus actions count ($g = 0.44$--$0.45$), and $D$ is the
rollout time per round, which contains the proposals' tool steps.

\paragraph{Decomposition of the speculative arm.}
For each round we record the wall-clock of every target verification and the
round's total verification wall-clock. The speculative arm's task time then
decomposes as
\[
  \text{spec} \;=\; \sum_{\text{rounds}} \max_{\text{used}} \text{wall}
  \;+\; \text{idle} \;+\; \text{draft} \;+\; \text{residual},
\]
where ``used'' are the verifications whose result the round consumed, idle is
$\sum_{\text{rounds}} (\text{verification wall} - \max_{\text{used}}
\text{wall})$, i.e.\ time spent waiting for verifications that were already
discarded, draft is the summed draft time, and residual covers tool execution
and bookkeeping. With prefix-ordered cancellation the idle time is 0\,s of
199\,s on OpenAGI, 0\,s of 400\,s on TaskBench, 3\,s of 4{,}367\,s on $\tau^2$-bench
and 14\,s of 11{,}132\,s on BFCL, and at most 4\,s in every other row of Table~\ref{tab:model}.
The earlier gather-mode runs waited for discarded verifications: 9\,s of 205\,s (OpenAGI), 12\,s of
391\,s (TaskBench) and 552\,s of 4{,}882\,s ($\tau^2$-bench) for the trained draft, 9\,s of 267\,s and 16\,s
of 481\,s for 32B-direct on OpenAGI and TaskBench, and 242\,s of 544\,s for the untrained draft.
Re-running every row with cancellation moved the main speedups from 1.46, 1.37 and 1.41\xx{} to 1.51, 1.37
and 1.58\xx{}, the untrained draft from 0.55\xx{} to 0.98\xx{}, 32B-direct from 1.12 and 1.11\xx{} to 1.16
and 1.13\xx{}, and OpenAGI at $k{=}1$ and $k{=}5$ from 1.30 and 1.48\xx{} to 1.34 and 1.57\xx{}. Residual
is at most 83\,s (BFCL). ``No idle'' in Table~\ref{tab:model} is
$\text{seq} / (\text{spec} - \text{idle})$.

\paragraph{Concurrency cost.}
To check assumption (b) of Sec.~\ref{sec:model} we compare milliseconds per
generated token in the sequential arm with the same quantity over the used
verifications of the speculative arm: 25.5 vs.\ 26.7 (OpenAGI), 25.5 vs.\ 26.3
(TaskBench), 26.0 vs.\ 26.7 ($\tau^2$-bench), 26.1 vs.\ 26.4 (BFCL). Tokens per step
are unchanged on three benchmarks and 5\% lower on $\tau^2$-bench.

\section{Serving configuration}
\label{app:serving}

All runs use vLLM on one NVIDIA H100 (80\,GB). The target is
Qwen3-32B (bf16, tensor parallel 1) and the draft is Qwen3-0.6B (bf16) with a
LoRA adapter of rank 64
served through vLLM's LoRA support (\texttt{--enable-lora},
\texttt{--max-lora-rank 64}); both servers run on the same GPU with the memory
fractions in Table~\ref{tab:serving}. Prefix caching is enabled on both
servers, and the cache is reset through the development endpoint
\texttt{/reset\_prefix\_cache} before each arm so that neither arm reuses the
other's KV cache. Requests are OpenAI-compatible chat completions with
temperature 0 and seed 0; $\tau^2$-bench and BFCL use the \texttt{xgrammar} guided
decoding backend and the \texttt{qwen3} reasoning parser.

\begin{table}[h]
\centering
\small
\caption{vLLM server settings per benchmark (target / draft on the same H100). A dash marks a setting
not used for that benchmark. On TaskBench at $k{=}5$ (Table~\ref{tab:model}) a round issues up to six
verification requests, so those beyond the fourth wait for a free slot.}
\label{tab:serving}
\begin{tabular}{lcccc}
\toprule
 & OpenAGI & TaskBench & $\tau^2$-bench & BFCL \\
\midrule
max model length (tokens) & 2{,}048 & 4{,}096 & 16{,}384 & 24{,}576 \\
GPU memory fraction, target / draft & 0.82 / 0.10 & 0.85 / 0.10 & 0.88 / 0.08 & 0.88 / 0.06 \\
max concurrent sequences, target / draft & 8 / 2 & 4 / 2 & 4 / 4 & 4 / 1 \\
guided decoding & -- & -- & xgrammar & xgrammar \\
reasoning parser & -- & -- & qwen3 & qwen3 \\
verification & \multicolumn{4}{c}{prefix-ordered, cancel} \\
\bottomrule
\end{tabular}
\end{table}

\paragraph{Common configurations of the size study.}
The rows of Table~\ref{tab:four} serve each benchmark's three drafts in one configuration that fits the 4B
draft beside the target: target / draft memory fractions 0.82 / 0.14 on OpenAGI and TaskBench, 0.83 / 0.13
with a 12{,}288-token context on $\tau^2$-bench, and 0.84 / 0.135 with a 15{,}360-token context on BFCL, all
with prefix-ordered cancellation; the Gemma target uses the same configurations.

\paragraph{The Gemma target.} Gemma-4-31B-it runs with FP8 weights on one H100: in bf16 it leaves a KV cache of
7{,}500 tokens at 85\% of the GPU, too few for $k{+}1$ concurrent verifications beside a draft. Its request
settings are those of the Qwen3 runs. Gemma departs from the prescribed action syntax more often than Qwen3
(Markdown fences, its native tool-call markup, keyword-argument calls on TaskBench, several calls in one
decision on $\tau^2$-bench); a surface-syntax normalization rewrites these forms into the prescribed one
before parsing, is applied identically to collection and to both arms, and every rewrite is counted (33 of
the 949 parsed actions of a TaskBench run).

\paragraph{Verification requests.}
The $k{+}1$ verification contexts of a round are issued as concurrent HTTP
requests to the target server. In the \emph{gather} mode of our earlier runs (no row of Table~\ref{tab:model} uses it), the round waits for all $k{+}1$ responses before
committing; the time spent waiting for responses that the first mismatch has
already made irrelevant is the idle term of App.~\ref{app:predictor}. In the \emph{prefix-ordered} mode used for every run we report, responses are consumed
in prefix order; as soon as context $i$ mismatches, the client closes the connections of
contexts $i{+}1,\dots,k$, which makes vLLM abort those requests, and the round
commits without waiting for them. The mode changes only when the round returns,
not what it commits.

\paragraph{Timing protocol.}
The test tasks of a shard are run first sequentially (target only) and then speculatively
on the same node and the same servers, with a cache reset in between; the
reported speedup is the pooled ratio of the summed whole-task wall-clocks.
OpenAGI was additionally run with prefix caching disabled on both servers as a
control (1.35\xx{} vs.\ 1.46\xx{} with caching, in the earlier gather-mode run). In Table~\ref{tab:four}, the
Qwen3 OpenAGI and TaskBench rows pool two independent passes over the test tasks (both passes of a task are
attached to it in the bootstrap); $\tau^2$-bench is run in two shards of 20 tasks and BFCL in three shards of
20, each sequential then speculative; the Gemma rows are single passes. The main runs of Table~\ref{tab:model}
were repeated once on $\tau^2$-bench and BFCL (App.~\ref{app:success}).

\section{Segment truncation by task length}
\label{app:length}
\begin{table}[h]
\centering
\caption{\textbf{Segment truncation by task length.} Tasks of the main runs bucketed by their sequential
length. $e$ (Eq.~\ref{eq:c}), the hit-rate sum $\sum_j q_j$ and $c$ are measured per bucket; the
prediction replays the bucket's sequential traces with the run-wide hit rates truncated at the task end,
so it carries the truncation effect only. Buckets hold 7--27 tasks, so individual ratios are noisy.}
\label{tab:length}
\small
\begin{tabular}{llrrrrrr}
\toprule
 & & & & & & \multicolumn{2}{c}{speedup} \\
\cmidrule(lr){7-8}
benchmark & steps & tasks & $e$ & $\sum_j q_j$ & $c$ & meas. & pred. \\
\midrule
OpenAGI   & 3        & 15 & 0.52 & 1.66 & 2.14 & 1.54\xx{} & 1.35\xx{} \\
OpenAGI   & 4        & 21 & 0.31 & 1.46 & 2.15 & 1.52\xx{} & 1.39\xx{} \\
OpenAGI   & 5        & 15 & 0.40 & 1.54 & 2.14 & 1.60\xx{} & 1.57\xx{} \\
OpenAGI   & $\ge 6$  &  9 & 0.30 & 1.10 & 1.80 & 1.39\xx{} & 1.55\xx{} \\
\midrule
TaskBench & $\le 4$  & 27 & 0.38 & 1.29 & 1.91 & 1.48\xx{} & 1.37\xx{} \\
TaskBench & 5        & 16 & 0.40 & 1.40 & 2.00 & 1.49\xx{} & 1.52\xx{} \\
TaskBench & 6        &  7 & 0.32 & 1.53 & 2.21 & 1.50\xx{} & 1.50\xx{} \\
TaskBench & $\ge 7$  & 10 & 0.12 & 0.77 & 1.65 & 1.14\xx{} & 1.52\xx{} \\
\midrule
$\tau^2$-bench   & $\le 9$  &  9 & 0.20 & 0.95 & 1.75 & 1.47\xx{} & 1.28\xx{} \\
$\tau^2$-bench   & 10--13   &  9 & 0.12 & 1.08 & 1.96 & 1.65\xx{} & 1.51\xx{} \\
$\tau^2$-bench   & 14--17   & 10 & 0.11 & 0.98 & 1.87 & 1.66\xx{} & 1.70\xx{} \\
$\tau^2$-bench   & $\ge 18$ & 12 & 0.07 & 0.96 & 1.89 & 1.55\xx{} & 1.72\xx{} \\
\bottomrule
\end{tabular}
\end{table}

Table~\ref{tab:length} buckets the tasks of the main runs by length. The truncation loss $e$ behaves as
Eq.~\ref{eq:c} says: it falls from 0.52 on three-step OpenAGI tasks to 0.30 on tasks of six or more
steps, and from 0.20 to 0.07 across $\tau^2$-bench's length range. The per-bucket speedup does not fall with
it, because task length and acceptance are not independent: short OpenAGI tasks are also the ones the
draft predicts best ($\sum_j q_j$ 1.66 against 1.10 on the longest tasks), and the longest TaskBench
tasks are where acceptance collapses ($\sum_j q_j$ 0.77, $c$ 1.65, 1.14\xx{}). A prediction that replays
each bucket with the run-wide hit rates captures the truncation effect but not this correlation, and
misses by up to 0.4 in the extreme buckets, whereas the run-level prediction is within 0.10. The
truncation term is therefore visible in the data, and the benchmarks' short tasks cost less than $e$
alone suggests because the draft is most often right on them.

\section{Transfer across benchmarks}
\label{app:transfer}
Table~\ref{tab:transfer} scores each trained 0.6B draft, and the untrained base, on the held-out contexts of
every benchmark with that benchmark's evaluator (Table~\ref{tab:offline}). All twenty cells come from one
job. The diagonal and the base row were measured again there; they differ from Table~\ref{tab:offline} by at
most 0.5 points, as repeated evaluations of the same draft do (the size study measured 75.4 and 74.7,
Table~\ref{tab:size}).

\begin{table}[h]
\centering
\caption{\textbf{Transfer across benchmarks.} Offline agreement (\%) of each 0.6B draft (rows) on the
held-out contexts of each benchmark (columns), scored as in Table~\ref{tab:offline} (BFCL: decision agreement
on the 99 development contexts).}
\label{tab:transfer}
\small
\begin{tabular}{lrrrr}
\toprule
 & \multicolumn{4}{c}{evaluated on} \\
\cmidrule(lr){2-5}
draft trained on & OpenAGI & TaskBench & $\tau^2$-bench & BFCL \\
\midrule
OpenAGI        & \textbf{75.5} & 29.6 & 8.8 & 3.0 \\
TaskBench      & 23.3 & \textbf{74.7} & 9.3 & 6.1 \\
$\tau^2$-bench & 10.8 & 21.7 & \textbf{63.7} & 1.0 \\
BFCL           & 12.0 & 23.1 & 8.0 & \textbf{66.7} \\
\midrule
untrained base & 13.3 & 25.3 & 8.0 & 2.0 \\
\bottomrule
\end{tabular}
\end{table}

\section{Online training details}
\label{app:online}
\paragraph{Stream and servers.} The training tasks of the benchmark are shuffled once with a fixed seed and
run in batches through the speculative system (target and draft on one GPU, the serving configuration of
App.~\ref{app:serving}). On OpenAGI each batch runs both arms. On $\tau^2$-bench the stream runs the
speculative arm only; the sequential arm on the same training tasks was run separately in the same serving
configuration. Fig.~\ref{fig:online}b reports, for both benchmarks, sequential time per step over speculative time
per committed step on each batch. Training runs on a second GPU while the servers wait between batches. The new adapter is
loaded into the running draft server by name through its runtime adapter endpoint, and the next batch proposes
with it. Every measurement on the stream is prequential: batch $b$ is never seen by the adapter that serves it.

\paragraph{Pairs.} On OpenAGI the commits of a batch are the (context, target action) pairs directly, with
duplicate (task, step) entries removed. $\tau^2$-bench records keep only hashes of the prompts, so the pairs are
rebuilt by replaying the committed actions in the benchmark environment, exactly as the collection pass
builds a step; every rebuilt prompt is checked against the hash the run recorded for the draft's request at
that step (898 pairs, no mismatch).

\paragraph{Recipe.} The adapter of every version is trained from the base with the recipe of
Table~\ref{tab:recipe}. Checkpoint selection uses a task-level 10\% carve of the buffer on OpenAGI and the
frozen 14 validation tasks on $\tau^2$-bench. The $\tau^2$-bench validation pairs were collected with the target sampling at
temperature 0.6, while the commits of the running system come from the verification setting (temperature 0);
the agreement of every $\tau^2$-bench draft in this paper, offline or online, is measured on the same validation
labels, so the curves are comparable, and the end-to-end acceptance is the measure under the verification
setting.

\paragraph{Continual variant.} Each version continues the previous adapter for two epochs on the new batch's
pairs plus an equal-size uniform sample of the earlier pairs, with the learning rate and rank of the recipe;
the first version is a fresh adapter on the first batch. The replay ratio and the number of epochs were not
tuned.

\input{tab_online_ci}

\paragraph{Endpoint.} The final adapter is measured on the test split with the protocol of
Table~\ref{tab:four} ($\tau^2$-bench: two shards of 20 tasks, each sequential then speculative, on separate server
pairs). Table~\ref{tab:online} lists the agreement of every version and the update times.

\begin{table}[h]
\centering
\caption{\textbf{Online training: agreement of each adapter version} (offline exact match on the held-out
contexts, \%). OpenAGI: 263 test contexts, one version per 25 training tasks; $\tau^2$-bench: 226 validation
contexts, one version per 10 tasks, so its 60 training tasks give six versions. Offline drafts: 75.0 (OpenAGI) and 63.7 ($\tau^2$-bench). Update time is the
time from the end of a batch to the loaded adapter (buffer build and training).}
\label{tab:online}
\small
\resizebox{\linewidth}{!}{\begin{tabular}{lrrrrrrrrrrr}
\toprule
version & 1 & 2 & 3 & 4 & 5 & 6 & 7 & 8 & 9 & 10 & update time \\
\midrule
OpenAGI, retrain on buffer & 53.0 & 58.8 & 67.0 & 69.6 & 73.3 & 74.7 & 73.4 & 72.8 & 74.1 & 74.0 & 27--73 s \\
OpenAGI, continual         & 48.7 & 61.6 & 65.7 & 69.8 & 63.5 & 69.2 & 70.8 & 72.5 & 75.9 & 74.9 & 22 s \\
$\tau^2$-bench, retrain on buffer & 33.6 & 57.5 & 58.4 & 58.4 & 64.6 & 66.4 & -- & -- & -- & -- & 3--18 min \\
$\tau^2$-bench, continual         & 28.3 & 57.1 & 54.9 & 59.7 & 63.3 & 63.3 & -- & -- & -- & -- & 3--6 min \\
\bottomrule
\end{tabular}}
\end{table}

\section{Task success and trajectory identity}
\label{app:success}
Table~\ref{tab:success} lists task outcomes for every configuration in Table~\ref{tab:four}. $\tau^2$-bench and BFCL use the benchmark's official success; OpenAGI and TaskBench have no
task-level success metric beyond termination, and the count is the number of runs that terminated. The
Qwen3 OpenAGI and TaskBench counts are over two runs of 63 and 60 tasks. The two passes of the Qwen3
$\tau^2$-bench main run gave 22 and 22 sequential against 21 and 20 speculative successes, and the two passes of the
BFCL main run 34 and 34 against 37 and 35. Across the 12 tool-benchmark size configurations and these four repeat comparisons,
official task success increases in 8 cases, decreases in 5, and is unchanged in 3.
Differences are at most three tasks per comparison. The 12 planning-benchmark comparisons instead
count normal termination: 3 decrease and 9 are unchanged.

Table~\ref{tab:parity} compares whole trajectories. The target run alone is reproducible: every pair of
target-only runs of the same configuration commits identical decisions on every task. Speculative runs are not,
mainly because verification decodes several contexts together and the engine's floating-point results depend on which
requests share a batch. On the planning benchmarks, where a target step takes 1--2\,s, 83--96\% of the
speculative trajectories are identical to the target-only one. On the tool benchmarks, where a step reasons for
10--12\,s, a small numerical difference changes some decision on most tasks, and the trajectory differs from
then on; two speculative runs also differ from each other. Every one of the 938 speculative task runs in
Table~\ref{tab:parity} passed the harness's commit audit, which replays the task and checks that each committed
action is the target's output on the context in which it was committed.

\paragraph{Serialized verification.} To separate batching from other effects, we reran speculation with the
target server limited to one sequence at a time, so that every verification request is decoded alone, as in the
target-only run (TaskBench: 60 tasks, both arms; $\tau^2$-bench: the first 20 test tasks, compared with the
target-only arm of the main run). The limit does not change the target-only run (60 of 60 TaskBench tasks
identical). The speculative trajectories then match the target-only ones on 57 of 60 TaskBench tasks, against 50
of 60 with concurrent verification, and on 8 of 20 $\tau^2$-bench tasks, against 0 of 20; every commit passes the
audit. Batching thus accounts for most of the divergence but not for all of it; we did not identify the other
source within the serving engine.

\begin{table}[!ht]
\centering
\caption{\textbf{Task outcomes per arm} for the configurations of Table~\ref{tab:four}.
Counts denote normal termination on OpenAGI and TaskBench, and official task success on
$\tau^2$-bench and BFCL. The sequential count is the same across the three draft-size comparisons.}
\label{tab:success}
\small
\begin{tabular}{lrrrrrrrr}
\toprule
 & \multicolumn{4}{c}{Qwen3-32B target} & \multicolumn{4}{c}{Gemma-4-31B target} \\
\cmidrule(lr){2-5}\cmidrule(lr){6-9}
 & & \multicolumn{3}{c}{speculative} & & \multicolumn{3}{c}{speculative} \\
\cmidrule(lr){3-5}\cmidrule(lr){7-9}
benchmark & sequential & 0.6B & 1.7B & 4B & sequential & 0.6B & 1.7B & 4B \\
\midrule
OpenAGI         & 126 & 126 & 126 & 126 & 63 & 63 & 63 & 63 \\
TaskBench       & 116 & 114 & 115 & 114 & 60 & 60 & 60 & 60 \\
$\tau^2$-bench  &  22 &  23 &  24 &  21 & 39 & 40 & 40 & 40 \\
BFCL multi-turn &  34 &  34 &  36 &  34 & 48 & 47 & 47 & 48 \\
\bottomrule
\end{tabular}

\vspace{1.5em}
\caption{\textbf{Trajectory identity.} Share of tasks (\%) on which two runs of the same configuration commit
identical decisions, with the free text of replies and summaries ignored; comparing that text as well gives
3.3\% on BFCL and 0\% on $\tau^2$-bench for target-only against speculative. OpenAGI and TaskBench: the six
runs of the draft-size study (three drafts, two repeats); $\tau^2$-bench and BFCL: the two passes of the main
configuration.}
\label{tab:parity}
\small
\begin{tabular}{lrrr}
\toprule
benchmark & two target-only runs & target-only vs.\ speculative & two speculative runs \\
\midrule
OpenAGI        & 100 & 95.5 & 96.5 \\
TaskBench      & 100 & 82.8 & 81.7 \\
$\tau^2$-bench & 100 &  1.3 & 45.0 \\
BFCL           & 100 & 34.2 & 60.0 \\
\bottomrule
\end{tabular}
\end{table}

%% file: tab_online_ci.tex
\begin{table}[tbp]
\centering\small\setlength{\tabcolsep}{8pt}
\caption{\textbf{Final online adapters on the test split, Qwen3 target.} Whole-task inference speedup
and task-bootstrap 95\% intervals, using trained 0.6B drafts with matched offline and online serving settings. Training time is reported separately.}
\label{tab:online_final_ci}
\begin{tabular}{lccc}
\toprule
Benchmark & Offline & Online: retrain & Online: continual \\
\midrule
OpenAGI & \shortstack{1.51\xx{}\\{\fontsize{8}{9}\selectfont [1.38, 1.64]}} & \shortstack{1.47\xx{}\\{\fontsize{8}{9}\selectfont [1.31, 1.63]}} & \shortstack{1.51\xx{}\\{\fontsize{8}{9}\selectfont [1.40, 1.64]}} \\
$\tau^2$-bench & \shortstack{1.61\xx{}\\{\fontsize{8}{9}\selectfont [1.44, 1.79]}} & \shortstack{1.61\xx{}\\{\fontsize{8}{9}\selectfont [1.43, 1.81]}} & \shortstack{1.63\xx{}\\{\fontsize{8}{9}\selectfont [1.46, 1.83]}} \\
\bottomrule
\end{tabular}
\end{table}